\documentclass[twoside]{article}

\usepackage[accepted]{aistats2025}
\usepackage[round]{natbib}

\usepackage{algorithm}
\usepackage{algorithmic}

\usepackage{latexsym}
\usepackage{amssymb}
\usepackage{amsmath}
\usepackage{amsthm}
\usepackage{booktabs}
\usepackage{enumitem}
\usepackage{graphicx}
\usepackage{color}
\usepackage[table]{xcolor}
\definecolor{poliblue3}{RGB}{25,43,67}

\usepackage{dsfont}
\usepackage{thmtools}
\usepackage{thm-restate}

\usepackage{etoolbox}
\usepackage{bm}
\usepackage{mathabx}
\usepackage{caption}

\newtheorem{theorem}{Theorem}[section]
\newtheorem{lemma}[theorem]{Lemma}
\newtheorem{corollary}[theorem]{Corollary}
\newtheorem{proposition}[theorem]{Proposition}

\theoremstyle{remark}
\newtheorem{remark}[theorem]{Remark}

\usepackage{tabularx}
\usepackage{multirow}
\usepackage{makecell}
\usepackage{arydshln} 

\usepackage{pifont}

\newcommand{\cmark}{\textcolor{green}{\ding{51}}} 
\newcommand{\xmark}{\textcolor{red}{\ding{55}}} 

\begin{document}

%

%

\twocolumn[

\aistatstitle{Achieving \(\widetilde{\mathcal{O}}(\sqrt{T})\) Regret in Average-Reward POMDPs with Known Observation Models}

\aistatsauthor{ Alessio Russo \And Alberto Maria Metelli \And  Marcello Restelli }

\aistatsaddress{ Politecnico di Milano \And Politecnico di Milano \And Politecnico di Milano  } ]

\begin{abstract}
We tackle average-reward infinite-horizon POMDPs with an unknown transition model but a known observation model, a setting that has been previously addressed in two limiting ways: (i) frequentist methods relying on suboptimal stochastic policies having a minimum probability of choosing each action, and (ii) Bayesian approaches employing the optimal policy class but requiring strong assumptions about the consistency of employed estimators. Our work removes these limitations by proving convenient estimation guarantees for the transition model and introducing an optimistic algorithm that leverages the optimal class of deterministic belief-based policies. We introduce modifications to existing estimation techniques providing theoretical guarantees separately for each estimated action transition matrix. Unlike existing estimation methods that are unable to use samples from different policies, we present a novel and simple estimator that overcomes this barrier. This new data-efficient technique, combined with the proposed \emph{Action-wise OAS-UCRL} algorithm and a tighter theoretical analysis, leads to the first approach enjoying a regret guarantee of order \(\mathcal{O}(\sqrt{T \,\log T})\) when compared against the optimal policy, thus improving over state of the art techniques. Finally, theoretical results are validated through numerical simulations showing the efficacy of our method against baseline methods.
\end{abstract}

\section{INTRODUCTION}

Reinforcement Learning (RL)~\citep{sutton1998} tackles the sequential decision-making problem of an agent interacting with an unknown or partially known environment with the goal of maximizing the long-term sum of rewards. The RL agent should trade-off between \emph{exploring} the environment to learn its structure and \emph{exploiting} the estimates to compute a policy that maximizes the reward. This problem has been successfully addressed in past works under the MDP formulation~\citep{bartlett2009regal, jaksch2010near, zanette2019tighter}. MDPs assume full observability of the state space but this assumption is often violated in many real-world scenarios such as robotics or finance, where only a partial observation of the environment is available. In this case, it is more appropriate to model the problem using Partially-Observable MDPs~\citep{sondik1978optimal}.

Further challenges emerge when tackling POMDPs since (i) the estimation problem turns into identifying the latent parameters of the model, (ii) the planning problem is known to be computationally intractable even for known models~\citep{mossel2005learning}.\\
Different approaches have been devised to tackle the estimation problem~\citep{guo2016pac, xiong2022sublinear}. In the finite horizon setting,~\citet{jin2020sample} present a sample-efficient algorithm for the undercomplete POMDP setting, where the number of observations is larger than the number of states, while in the average-reward setting,~\citet{Azizzadenesheli2016Reinforcement} provide guarantees on the regret by introducing a model-based approach that leverages \emph{spectral decomposition}~\citep{Anandkumar2014Tensor} techniques to estimate the latent model while employing memoryless policies.

This paper considers an average-reward POMDP setting with a known observation model but an unknown transition model that the agent needs to learn. The assumption of having partial knowledge of the environment has been variously addressed in the past, both in the bandit setting~\citep{MaillardLatent2014, russo2024switching} and in the POMDP setting~\citep{jahromi2022online, russo2024efficient}. Relying on the knowledge of the observation model can be of interest in many real-world scenarios. Sometimes this knowledge is available from scratch. For example, in robotics, the properties of the sensors of a robot are available beforehand, while in other scenarios the observation model can be learned offline from simulation~\citep{thanan2021resource, alonso2017virtual} or from historical data, while the transition model of the environment needs to be learned in an online fashion.
This assumption is also reasonable in the case of non-stationary environments where the change is local and associated only with the transition model while the previous knowledge of the observation model can be retained. Similar motivations hold as well for problems dealing with \emph{Transfer Learning}. 

\paragraph{Contribution} We report here the main contributions of our work:
\begin{itemize}
    \item We present the \emph{Action-wise} OAS procedure that estimates the transition model under the assumption of knowing the observation model. We show that this technique also handles samples collected under different policies.
    \item Under some technical assumptions, we prove estimation guarantees separately for the transition matrix associated with each action. 
    \item We introduce the \emph{Action-wise} OAS-UCRL algorithm, an optimistic approach that makes use of the presented estimation method and employs the optimal class of deterministic belief-based policies. 
    \item By using new theoretical results, we prove that this algorithm enjoys a \(\mathcal{O}(\sqrt{T \,\log T})\) regret guarantee when compared against the optimal POMDP policy, thus improving over state-of-the-art results.
\end{itemize}

\section{RELATED WORKS}\label{sec:relatedWorks}
In recent years, significant progress has been made in understanding the fully observable RL setting. Some studies focused on the episodic scenario with finite horizon~\citep{jin2018qlearning, azar2017MinimaxRB}, while others have examined the 
non-episodic undiscounted setting~\citep{jaksch2010near, ortner2012e, bartlett2009regal}. In contrast, Partially Observable MDPs have received relatively less attention, also due to the inherent difficulty of this setting both from the learning and the planning perspective.

\paragraph{Learning in POMDPs} A POMDP instance is considered hard when the observation model does not contain enough information to allow learning the underlying transition model. These pathological cases can be avoided by assuming that the observation model is full-rank or, stated differently, when the minimum singular value \(\alpha\) of the observation model is positive, namely \(\alpha > 0\). Instances belonging to this class are called \textbf{\(\bm{\alpha}\)-weakly revealing} and can be learned efficiently.

Within this class of tractable problems, some works focused on the \textbf{episodic} setting such as~\citet{jin2020sample} and~\citet{liu2022partially}. The first one considers the undercomplete case, with the number of states less or equal to the number of observations \((S \le 0)\): the authors do not focus on regret but introduce an algorithm with optimal sample complexity (\(1/\epsilon^2\)) for learning an \(\epsilon\)-optimal policy. Differently,~\citet{liu2022partially} provide an approach based on a \emph{Maximum Likelihood Estimation} technique and show regret guarantees of order \(\widetilde{\mathcal{O}}(\sqrt{T})\) for the undercomplete case. They also consider the more difficult overcomplete setting \((S > O)\) for which they prove a guarantee of order \(\widetilde{\mathcal{O}}(T^{2/3})\). In the subsequent work of~\citet{chen2023lower}, these regret guarantees were shown to be tight with respect to the horizon \(T\) in both the undercomplete and overcomplete settings.

A second stream of works focuses on the \textbf{non-episodic average reward} setting. In particular,~\citet{Azizzadenesheli2016Reinforcement} and~\citet{xiong2022sublinear} consider the standard POMDP setting where neither the observation nor the transition model are known and they both employ \emph{Spectral Decomposition} (SD) techniques to retrieve the model parameters.~\citet{Azizzadenesheli2016Reinforcement} consider the class of memoryless stochastic policies\footnote{In a memoryless policy, the next action is chosen only based on the last received observation.} with each action having a minimum probability \(\iota > 0\) of being chosen: they show a regret guarantee of order \(\widetilde{\mathcal{O}}(\sqrt{T})\) under this class of stochastic policies. Concerning the work of~\citet{xiong2022sublinear}, they present the SEEU algorithm which alternates between purely exploratory and purely exploitative phases. Their algorithm reaches a \(\mathcal{O}(T^{2/3})\) regret when compared against the optimal class of belief-based policies.

On the other hand, both~\citet{jahromi2022online} and~\citet{russo2024efficient} assume, as in our setting, to have \textbf{knowledge of the observation model}.~\citet{jahromi2022online} develop the PSRL-POMDP algorithm, a Bayesian approach that jointly learns the model parameters and exploits the available knowledge. They prove a Bayesian regret of order \(\mathcal{O}(T^{2/3})\) when compared against the optimal policy, however they do not provide guarantees for the employed model estimator: the obtained result on the regret only holds by assuming to have a consistent estimator. Differently,~\citet{russo2024efficient} develop the OAS estimation approach and prove consistency for it. They consider the powerful class of belief-based policies but focus on those having a minimum probability \(\iota > 0\) of choosing each action. Under these conditions, they reach a regret order of \(\widetilde{\mathcal{O}}(\sqrt{T})\) when compared against this class of stochastic policies.\\
We refer the reader to Appendix~\ref{appendix:comparison} and to Table~\ref{tab:comparisonRL} for a detailed comparison with the works mentioned above. It characterizes the different works in terms of adopted assumptions, applied estimation procedures, and algorithm properties.

Stemming from the OAS estimation procedure introduced in~\citet{russo2024efficient}, we adopt similar ideas and present a new estimation approach that overcomes the limiting assumption on the minimum action probability. This aspect allows us to successfully compare the newly devised regret minimization algorithm against the optimal class of POMDP policies.\\

\section{PRELIMINARIES}\label{sec:preliminaries}
In this section, we present the adopted notation and the necessary background that will be useful to understand what will follow.

\paragraph{Partially Observable MDP} A Partially Observable Markov Decision Process (POMDP)~\citep{astrom1965optimal} is defined by a tuple \(\mathcal{Q}\coloneq(\mathcal{S}, \mathcal{A}, \mathcal{O}, \mathbb{T}, \mathbb{O}, \bm{\nu}, r)\) with \(\mathcal{S}\) being a finite state space (\(|\mathcal{S}|\eqcolon S\)), \(\mathcal{A}\) a finite action space (\(|\mathcal{A}| \eqcolon A\)) and \(\mathcal{O}\) a finite observation space \((|\mathcal{O}|\eqcolon O)\). \(\mathbb{T}=\{\mathbb{T}_a\}_{a \in \mathcal{A}}\) denotes a collection of transition matrices \(\mathbb{T}_a\) with size \(S\times S\). Each transition matrix is such that \(\mathbb{T}_a(\cdot|s) \in \Delta(\mathcal{S})\)\footnote{We use \(\Delta(\mathcal{X})\) to denote the simplex over a finite space \(\mathcal{X}\).} defines the distribution of the next state when the agent takes action \(a\) in state \(s\). \(\mathbb{O}=\{\mathbb{O}_a\}_{a \in \mathcal{A}}\) denotes the set of observation matrices of size \(O \times S\) such that \(\mathbb{O}_a(\cdot|s) \in \Delta(\mathcal{O})\) represents the distribution over observations when the agent takes action \(a\) conditioned on the hidden state \(s\). \(\bm{\nu} \in \Delta(\mathcal{S})\) denotes the distribution over the initial state, while \(r: \mathcal{O} \rightarrow [0,1]\) is the known reward function mapping each observation to a finite reward such that \(r(o)\) is the reward received when the agents observe \(o \in \mathcal{O}\).
In a POMDP, states are hidden and the agent can only see its own actions and the received observations. The interaction proceeds as follows. At each step \(t\), the agent chooses an action \(a_t\) and gets an observation \(o_t\) which is conditioned on the hidden state \(s_t\) according to the law \(\mathbb{O}_{a_t}(\cdot|s_t)\). Finally, the action taken will make the POMDP transition into a new hidden state \(s_{t+1}\) according to the distribution \(\mathbb{T}_{a_t}(\cdot|s_t)\).

\paragraph{Policies in POMDPs} A policy \(\pi:= (\pi_t)_{t=0}^\infty\) defines the set of decision rules characterizing the interaction of an agent with the environment. Deterministic policies map a history \(h \in \mathcal{H}_t\) into actions \(\pi_t: \mathcal{H}_t \rightarrow \mathcal{A}\) such that \(a=\pi_t(h)\) defines the action chosen when history \(h \in \mathcal{H}_t\) is observed. When interacting with a POMDP the history can be defined as \(h:=(a_j,o_j)_{j=0}^t \in \mathcal{H}_t\). We denote by \(\mathcal{H}\) the space of histories of arbitrary length.

\paragraph{From POMDP to Belief MDP} In a POMDP, it is always possible to build a belief vector \(b_t \in \mathcal{B}\) (with \(\mathcal{B} \coloneq \Delta(\mathcal{S})\)) by using the knowledge of the transition and observation matrices \(\mathbb{T}\) and \(\mathbb{O}\) respectively, and from the observed history at time \(t-1\), \(h_{t-1}:= (a_j, o_j)_{j=0}^{t-1}\). Thus, at time \(t\) it holds \(b_t(s):=P(S_t=s|h_{t-1})\). From the agent's point of view, a POMDP can be seen as a belief MDP~\citep{krish2016partially}. The update rule of the belief \(b_{t+1}\) is determined by Bayes's theorem as follows:
\begin{align}\label{eq:beliefUpdate}
	b_{t+1}(s) = \frac{\sum_{s' \in \mathcal{S}}b_t(s') \mathbb{O}_{a_t}(O_t=o_t|S_t=s') \mathbb{T}_{a_t}(s|s')}{\sum_{s'' \in \mathcal{S}} \mathbb{O}_{a_t}(O_t=o_t|S_t=s'') b_t(s'')}.
\end{align}
Given an initial belief \(b\), the average reward of the infinite-horizon belief MDP  is defined as: \(\rho_b^\pi:=\lim \sup_{T \to \infty} (1/T)\mathbb{E}[\sum_{t=0}^{T-1}r(o_t)|b_0=b)]\). If the underlying MDP is weakly communicating,~\citet{bertsekas1995dynamic} showed that \(\rho^*:=\sup_\pi \rho_b^\pi\) is independent of the initial belief \(b\) and the following Bellman optimality equation can be defined:
\begin{align}\label{eq:Bellman}
	\rho^* + v(b) = \underset{a \in \mathcal{A}}{\max}\; \left[g(b,a) + \int_{\mathcal{B}}P(\,db'|b, a) v(b') \right],
\end{align}
with \(g(b,a)\) representing the expected instantaneous reward obtained when taking action \(a\) under belief \(b\) such that \(g(b,a) = \sum_{s \in \mathcal{S}} \sum_{o \in \mathcal{O}} b(s) \mathbb{O}_a(o|s)r(o)\). Ultimately, \(v: \mathcal{B} \to \mathbb{R}\) defines the bias function quantifying the cumulative deviation of rewards with respect to \(\rho^*\) when starting from a belief \(b\)~\citep{Mahadevan1996average}.

\section{PROBLEM FORMULATION}\label{sec:problemFormulation}
We tackle the average-reward infinite-horizon POMDP setting described in Section~\ref{sec:preliminaries}. 
In particular, we focus on the class of \emph{undercomplete} POMDPs~\citep{jin2020sample}, where the number of states is less than or equal to the number of observations \((S \le O)\). As in previous works~\citep{jahromi2022online, russo2024efficient}, we assume knowledge of the observation model \(\mathbb{O}=\{\mathbb{O}_a\}_{a \in \mathcal{A}}\), while we learn the transition model \(\mathbb{T}=\{\mathbb{T}_a\}_{a \in \mathcal{A}}\).\\
We consider the class of belief-based policies mapping the space \(\mathcal{B}\) of belief over the states to actions, such that \(\pi: \mathcal{B} \to \mathcal{A}\). We denote by \(\mathcal{P}\) the set of such belief-based policies.\\
In the following, we introduce the assumptions that we enforce for our setting.
\begin{restatable}{assumption}{assMinElem}\label{ass:minElem}
	\textup{(\textbf{Minimum Value Transition Matrices})} The smallest value in the transition matrices is \(\epsilon := \underset{s,s' \in \mathcal{S}\; a \in \mathcal{A}}{\min}\mathbb{T}_a(s'|s) > 0\).
\end{restatable}
Despite seeming a strong assumption, this one-step reachability condition is widely used in works addressing partial observability~\citep{zhou2021regime, russo2024efficient, jiang2023online, xiong2022sublinear}. It is used for multiple reasons: first, it ensures geometric ergodicity of the Markov chain induced by the employed policy; second, it plays a key role in the theoretical analysis since allows to bound the error in the estimated belief vector as a function of the error in the estimated transition model (see Lemma~\ref{lemma:BoundOnBeliefErrors} for details). In practical scenarios, this assumption is satisfied in various POMDP applications, such as those involving information gathering~\citep{guo2016pac}.

\begin{restatable}{assumption}{assWeaklyRev}\label{ass:weaklyRev}\textup{(\textbf{\(\bm{\alpha}\)-weakly Revealing Condition})}
	There exists \(\alpha>0\) such that \(\underset{a \in \mathcal{A}}{\min} \sigma_S(\mathbb{O}_a) \ge \alpha\).
\end{restatable}
Here, we use \(\sigma_S(\mathbb{O}_a)\) to denote the \(S\)-th singular value of matrix \(\mathbb{O}_a\).
This second assumption relates to the identifiability of the POMDP parameters and has been largely used in works tackling the partially observable setting. This condition quantifies the amount of information provided by the observations when inferring the latent states. 
A positive \(\alpha\) value rules out pathological POMDP instances and identifies the tractable subclass of \emph{weakly revealing} POMDPs~\citep{jin2020sample, liu2022partially, liu2022optimistic} (see Section~\ref{sec:relatedWorks}). This assumption is related to the more typical \emph{full-rank} condition employed in works using spectral decomposition techniques~\citep{Azizzadenesheli2016Reinforcement, zhou2021regime, hsu2012spectral}. A direct implication of this condition is that \(S \le O\), which represents a common scenario in many real-world settings, such as medical applications where the state (physical condition) of a patient generates a large number of different observations~\citep{hauskrecht2000PlanningTO} or dialogue systems, with a number of observations (words) that is much larger than the possible states (topics) of the conversation~\citep{png2012building}.

\paragraph{Learning Objective} 
As defined before, the objective in our setting is to find the belief-based policy maximizing Equation~\eqref{eq:Bellman}.\\
~\citet{xiong2022sublinear} have shown that under assumption~\ref{ass:minElem}, this equation is always verified.
We tackle this problem using a regret minimization approach and we compare our policy against the optimal POMDP policy which plays according to the current belief value. Since determining the optimal policy for the POMDP model is generally computationally intractable~\citep{madani1999computability}, in this work we do not focus on solving this planning problem. Instead, we assume access to an optimization oracle able to solve Equation~\eqref{eq:Bellman}, thus maximizing the average reward \(\rho^*\) while returning the optimal policy \(\pi \in \mathcal{P}\) under a given model.\\
The total regret over the interaction horizon of length \(T\) is thus defined as:
\begin{align}\label{def:regret}
	\mathcal{R}_T := T\rho^* - \sum_{t=0}^{T-1} r^\pi(o_t),
\end{align}
with apex \(\pi\) in the reward denoting that observations are obtained while following policy \(\pi \in \mathcal{P}\).

\section{ACTION-WISE OAS ESTIMATION PROCEDURE}\label{sec:estimationProcedure}
Based on the OAS approach developed in~\citet{russo2024efficient}, we present here the \emph{Action-wise Observation-Aware Spectral} (AOAS) estimation technique that aims at estimating the transition model \(\mathbb{T} = \{\mathbb{T}_a\}_{a \in \mathcal{A}}\) characterizing the POMDP\footnote{Refer to Appendix~\ref{appendix:comparisonOAS} for a detailed comparison between the approaches.}.\\
Let us assume to interact with a POMDP instance \(\mathcal{Q}\) using a belief-based policy \(\pi \in \mathcal{P}\) and to collect samples \(\mathcal{D}=\{(a_t,o_t)_{t=0}^{n}\}\) with \(n+1=|\mathcal{D}|\) denoting the cardinality of the dataset.\\
We will then group pairs of samples collected in consecutive timestamps such that from dataset \(\mathcal{D}\) we can build a new dataset \(\mathcal{G}= \{ (a_{t},a_{t+1},o_{t},o_{t+1})_{t=0}^{n-1}\}\) having cardinality \(n=|\mathcal{G}|\).
The tuples of the form \((a,a',o,o')\) with \(a,a' \in \mathcal{A}\) and \(o,o' \in \mathcal{O}\) will be the core elements employed in our estimation approach.

The estimation of transition matrix \(\mathbb{T}_a\) related to action \(a\) is done by only considering those tuples in \(\mathcal{G}\) whose first element \((a_{t})_{t=0}^{n-1}\) coincides with action \(a\), while the remaining part of the tuple \((a_{t+1},o_{t},o_{t+1})\) is actually employed for estimation. For convenience, we use a vector notation to represent the last three elements of each tuple, such that \(\bm{x}_t \in \mathbb{R}^{AO^2}\) will denote the one-hot encoding of tuple \((a_{t+1},o_{t},o_{t+1})\).

Let us now denote by \(\mathcal{E}(a, n, m)\) the event which holds true when, in a dataset \(\mathcal{G}\) of consecutive samples having cardinality \(n\), the number of tuples having action \(a\) as a first element is equal to \(m\). More formally:
\begin{align}\label{def:conditionalEvent}
	\mathcal{E}(a, n, m) = \left \{ m = \sum_{t=0}^{n-1}\mathds{1}\{a_{t}=a\} \right \}.
\end{align}
Here, we use \(\mathds{1}\{\cdot\}\) to denote the indicator function, which is equal to 1 when the condition is satisfied and 0 otherwise.\\
The elements defined above allow us to present the following distribution of interest \(\bm{d}^{(a,n,m)}_{AO^2} \in \Delta(\mathcal{A} \times \mathcal{O}^2)\) over the tuples \((a',o,o')\):
\begin{align}\label{def:conditionalDist}
	\bm{d}^{(a,n, m)}_{AO^2} = \mathbb{E}_{\pi, \bm{\nu}}\Bigg[\frac{1}{m} \sum_{t=0}^{n-1} \mathds{1}\{a_{t}=a\}\; \bm{x}_t\ \bigg| \mathcal{E}(a,n,m)\Bigg],
\end{align}
where the expectation is with respect to policy \(\pi \in \mathcal{P}\) and the initial state distribution \(\bm{\nu} \in \Delta(\mathcal{S})\) of the POMDP, and it is conditioned on the event \(\mathcal{E}(a,n,m)\). Here, the subscript \(AO^2\) employed for the presented distribution represents the size of its support.\footnote{A similar notation will be used as well for other distributions defined throughout the work.}

Since we know that the received observations can be mapped to the underlying latent states using the observation model \(\mathbb{O}\), we can define a relation linking the distribution \(\bm{d}^{(a,n,m)}_{AO^2}\) defined on the tuples \((a',o,o')\) with an analogous distribution \(\bm{d}^{(a,n,m)}_{AS^2} \in \Delta(\mathcal{A} \times \mathcal{S}^2)\) defined on the non-observable tuples \((a',s,s')\)\footnote{For a formal definition of this distribution, we refer to Appendix~\ref{appendix:lemmaProof}.
}. Indeed, we can easily observe that, for each element \((a',o,o')\) of vector \(\bm{d}^{(a,n,m)}_{AO^2}\), we have:
\begin{align*}
	& \bm{d}^{(a,n,m)}_{AO^2}(a',o,o') = \notag \\
	& \qquad \sum_{s,s' \in \mathcal{S}} \mathbb{O}_a(o|s) \, \mathbb{O}_{a'}(o'|s') \; \bm{d}^{(a,n,m)}_{AS^2}(a',s,s').
\end{align*}
The relation stated above can be defined for all the elements of the considered distributions. Thus, using matrix notation, we have:
\begin{align}\label{eq:conditionalAO^2->AS^2}
	\bm{d}^{(a,n,m)}_{AO^2} = \mathbb{B}_a \; \bm{d}^{(a,n,m)}_{AS^2},
\end{align}
where \(\mathbb{B}_a\) is a block diagonal matrix of size \(AO^2 \times AS^2\) obtained by aligning along its diagonal the matrices \(\{\mathbb{O}_{a,a'}\}_{a' \in \mathcal{A}}\).
The different matrices \(\mathbb{O}_{a,a'}\) have dimension \(O^2 \times S^2\) and are in turn obtained as follows:
\begin{align*}
	\mathbb{O}_{a,a'} \coloneq \mathbb{O}_a \otimes \mathbb{O}_{a'},
\end{align*}
where \(\otimes\) denotes the Kronecker product between matrices \(\mathbb{O}_a\) and \(\mathbb{O}_{a'}\). We recall that since the observation model is available, we can compute the block diagonal matrix \(\mathbb{B}_a\) for any \(a \in \mathcal{A}\).

The distribution \(\bm{d}^{(a,n,m)}_{AS^2}\) can be linked to the transition matrix \(\mathbb{T}_a\) using the following considerations. First of all, we define a new quantity \(\bm{d}_{S^2}^{(a,n,m)} \in \Delta(\mathcal{S}^2)\) which is obtained by aggregating elements in \(\bm{d}^{(a,n,m)}_{AS^2}\). In particular, each element of this new vector is obtained as:
\begin{align}\label{eq:fromASStoSS}
\bm{d}_{S^2}^{(a, n,m)}(s,s') = \sum_{a' \in \mathcal{A}} \bm{d}_{AS^2}^{(a, n,m)}(a',s,s').
\end{align}
The final step involves recognizing the proportional relationship between elements in \(\mathbb{T}_a\) and \(\bm{d}_{S^2}^{(a,n,m)}\), which leads to the final expression:
\begin{align}\label{eq:TfromD}
	\mathbb{T}_a(s'|s) = \frac{\bm{d}_{S^2}^{(a,n,m)}(s,s')}{\sum_{s'' \in \mathcal{S}}\bm{d}_{S^2}^{(a,n,m)}(s,s'')} \quad \forall s,s' \in \mathcal{S}. 
\end{align}
For details on the derivation of Equation~\eqref{eq:TfromD}, we refer to Lemma~\ref{lemma:linkTransitionModelAndDist}.

\paragraph{Estimation Procedure} Having defined the procedure connecting the distribution on the action-observation tuples \(\bm{d}^{(a,n,m)}_{AO^2}\) to the transition model, we show how an estimate of the transition matrix \(\mathbb{T}_a\) can be computed. The following holds for any action \(a \in \mathcal{A}\). Let a policy \(\pi\) interact with the environment for \(n+1\) timestamps generating a dataset \(\mathcal{D}\) of samples and let us group consecutive samples obtaining a new dataset \(\mathcal{G}\) with cardinality \(n\), as previously described. By denoting with \(n(a)\) the number of tuples in \(\mathcal{G}\) starting with action \(a\), we can estimate \(\bm{d}^{(a,n,n(a))}_{AO^2}\) as: 
\begin{align}\label{def:singleEpEstimator}
	\widehat{\bm{d}}^{(a,n,n(a))}_{AO^2} = \frac{1}{n(a)} \sum_{t=0}^{n-1} \mathds{1}\{a_{t}=a\} \; \bm{x}_t.
\end{align}
The estimate corresponding to the associated distribution \(\widehat{\bm{d}}^{(a,n,n(a))}_{AS^2}\) over the non-observable tuples \((a',s,s')\) can thus be obtained by inverting Equation~\eqref{eq:conditionalAO^2->AS^2} as follows:
\begin{align}\label{eq:invertedToConditionalASS}
	\widehat{\bm{d}}^{(a,n,n(a))}_{AS^2} = \mathbb{B}_a^\dagger \; \widehat{\bm{d}}^{(a,n,n(a))}_{AO^2},
\end{align}
where \(\mathbb{B}_a^\dagger\) denotes the Moore-Penrose of matrix \(\mathbb{B}_a\). We stress that, by the weakly-revealing assumption (\ref{ass:weaklyRev}) and the properties of the Kronecker product, this matrix is always invertible since it holds that \(\sigma_{\min}(\mathbb{B}_a) \ge \alpha^2\).\\
In a subsequent step, Equation~\eqref{eq:fromASStoSS} is used to obtain an estimate \(\widehat{\bm{d}}^{(a,n,n(a))}_{S^2}\). Since this estimate may erroneously contain negative terms, we modify the negative ones by setting them to \(0\), thus obtaining a non-negative vector \(\widebar{\bm{d}}^{(a,n,n(a))}_{S^2}\). The newly obtained quantity is then plugged into Equation~\eqref{eq:TfromD} to compute an estimate \(\widehat{\mathbb{T}}_a\) of the action transition matrix.\\
The pseudocode of the reported approach is presented in Algorithm~\ref{alg:AOASEstimation}.

\subsection{Sample Reuse and Theoretical Guarantees}
One of the typical issues affecting the learning of model parameters in the average-reward POMDP setting is the inability to use samples coming from different policies~\citep{Azizzadenesheli2016Reinforcement, russo2024efficient}. 
The objective of this section is to show a simple estimator that is able to overcome this problem.\\ 
In particular, let us assume that \((\pi_i)_{i=0}^{k-1}\) policies interact separately with the environment and let us denote with \((\mathcal{G}_i)_{i=0}^{k-1}\) the generated datasets of consecutive samples. 
Let \(n_i\) and \(n_i(a)\) indicate respectively the cardinality of the dataset \(\mathcal{G}_i\) and the number of tuples from this dataset starting with action \(a\). Based on these quantities, we can define a mixed distribution described as:
\begin{align}\label{eq:combinedDistribution}
	\bm{d}_{AO^2}^{(a,k)} = \frac{1}{N_k(a)} \sum_{i=0}^{k-1} n_i(a) \; \bm{d}_{AO^2}^{(a,n_i,n_i(a))},
\end{align}
where \(N_k(a)=\sum_{i=0}^{k-1} n_i(a)\), while the new quantity \(\bm{d}_{AO^2}^{(a,k)} \in \Delta(\mathcal{A} \times \mathcal{O}^2)\) mixes the per-policy distributions assigning to each of them a weight proportional to \(n_i(a)\). An unbiased estimate of this quantity is obtained as:
\begin{align}\label{eq:combinedEstimator}
		\widehat{\bm{d}}_{AO^2}^{(a,k)} & = \frac{1}{N_k(a)}\sum_{i=0}^{k-1} n_i(a) \;\widehat{\bm{d}}_{AO^2}^{(a,n_i,n_i(a))} \notag \\ 
		& = \frac{1}{N_k(a)}\sum_{i=0}^{k-1} \sum_{j=0}^{n_i-1} \mathds{1}\{a_{j,i}=a\} \; \bm{x}_{j,i}
\end{align}
where we use \(a_{j,i}\) to represent the action at timestamp \(j\) referring to dataset \(\mathcal{G}_i\), while \(\bm{x}_{j,i}\) denotes the \(j\)-th indicator vector from \(\mathcal{G}_i\).\\ By observing Equation~\eqref{eq:combinedEstimator}, we can see that it is equivalent to the estimator defined in~\eqref{def:singleEpEstimator} when computed on the unique dataset \(\mathcal{U}= \bigcup_{i=0}^{k-1}\,\mathcal{G}_i\) obtained from the union of the different datasets \(\mathcal{G}_i\).\\
As we will see in Lemma~\ref{lemma:combinedActionDistr}, the number \(k\) of different policies influences the guarantees of the estimated transition matrix. However, this aspect is not reflected in the pseudocode of the \emph{Action-wise} OAS procedure where the approach can be used without modifications by simply providing as input the union dataset \(\mathcal{U}\). 

\begin{algorithm}[tb]
	\caption{Action-wise OAS Algorithm}
	\label{alg:AOASEstimation}
	\begin{algorithmic}[1]
		\STATE {\textbf{Input:}} Observation matrix \(\{\mathbb{O}_a\}_{a \in \mathcal{A}}\), dataset \(\mathcal{G}=\{(a_0,a_1,o_0,o_1), \dots, (a_{n-1},a_{n},o_{n-1},o_{n})\}\) of consecutive samples, dataset size \(n=|\mathcal{G}|\)
		\STATE Create block diagonal matrices \(\; \{\mathbb{B}_a\}_{a \in \mathcal{A}}\) from the observation model \(\mathbb{O}\)
		\STATE Set action counters \(n(a) = 0 \quad \forall a \in \mathcal{A}\)
		\STATE Define vectors of count \(\bm{c}(a) \in \mathbb{R}^{AO^2} \;\; \forall a \in \mathcal{A}\) and set their elements to zero
		\STATE t = 0
		\WHILE{\(t < n\)}
		\STATE Get tuple \((a_{t},a_{t+1},o_{t}, o_{t+1})\) from \(\mathcal{G}\)
		\STATE \(\bm{x}_t \leftarrow \text{one\_hot\_encode}(a_{t+1},o_{t}, o_{t+1}\)) 
		\STATE \(\bm{c}(a_t) = \bm{c}(a_t) + \bm{x}_t\)
		\STATE \(n(a_t) = n(a_t) + 1\)
		\STATE \(t=t+1\)
		\ENDWHILE
		\FOR{\(a \in \mathcal{A}\)}
		\IF{\(n(a) > 0\)}
		\STATE Compute \(\widehat{\bm{d}}^{(a,n,n(a))}_{AO^2} = \bm{c}(a) / n(a)\) 
		\STATE Compute \(\widehat{\bm{d}}^{(a,n,n(a))}_{AS^2}\) using Equation~\eqref{eq:invertedToConditionalASS}
		\STATE Compute \(\widehat{\bm{d}}^{(a,n,n(a))}_{S^2}\) using Equation~\eqref{eq:fromASStoSS}
		\STATE Compute positive \(\widebar{\bm{d}}^{(a,n,n(a))}_{S^2}\) from \(\widehat{\bm{d}}^{(a,n,n(a))}_{S^2}\)
		\STATE Compute \(\widehat{\mathbb{T}}_a\) from \(\widebar{\bm{d}}^{(a,n,n(a))}_{S^2}\) using Equation~\eqref{eq:TfromD}
		\ENDIF
		\ENDFOR
	\end{algorithmic}
\end{algorithm}

The mixed distribution defined in~\eqref{eq:combinedDistribution} and its estimator show how to combine samples from different policies. By employing these quantities in the analysis and using Algorithm~\ref{alg:AOASEstimation} on the collected data, we prove a consistent approach for estimating each action transition matrix \(\mathbb{T}_a\), as observed from the following result: 

\begin{restatable}[]{lemma}{combinedActionDistr}\label{lemma:combinedActionDistr}
Let us assume that \(k\) policies \((\pi_i)_{i = 0}^{k-1}\), each with \(\pi_i \in \mathcal{P}\), separately interact with a POMDP instance \(\mathcal{Q}\) satisfying Assumptions~\ref{ass:minElem} and~\ref{ass:weaklyRev}. By providing the union dataset \(\mathcal{U} = \bigcup_{i=0}^{k-1}\,\mathcal{G}_i\) to Algorithm~\ref{alg:AOASEstimation}, with probability at least \(1 - \delta\), it holds that: 
\begin{equation*}
	\|\mathbb{T}_a - \widehat{\mathbb{T}}_a\|_{F} \le \frac{4 \widetilde{G}}{\alpha^2 d_{\min}^{(a)}\; (1 - \widetilde{\eta})}\sqrt{\frac{2k\,SA \,\log (2AO^2k/\delta)}{N_k(a)}}
\end{equation*}
where \(\widetilde{G} \ge 1\) and \(\widetilde{\eta} \le 1 - \frac{\epsilon}{1 - \epsilon}\) are determined by the deployed policies, while \(d_{\min}^{(a)}\) represents the minimum state distribution conditioned on action \(a\). 
\end{restatable}
The parameter \(\widetilde{\eta}\) appearing in the bound refers to a contraction coefficient associated with the different Markov chains induced by the policies and its value is always strictly smaller than 1 under Assumption~\ref{ass:minElem}. Using this assumption, we are also able to bound away from 0 the minimum state distribution \(d_{\min}^{(a)}\). Finally, the \(\alpha\) term deriving from Assumption~\ref{ass:weaklyRev} characterizes the amount of information carried by the observations to infer the underlying states.

Furthermore, we remark that Lemma~\ref{lemma:combinedActionDistr} remains valid under a weaker condition than Assumption~\ref{ass:minElem}. In particular, it suffices to impose an ergodicity-like assumption for each action. In Appendix~\ref{appendix:lemmaProof}, we provide a more detailed description of this assumption together with a formal derivation of the results of the Lemma.

\begin{algorithm}[tb]
	\caption{The Action-wise OAS-UCRL Algorithm}
	\label{alg:oasAlgorithm}
	\begin{algorithmic}[1]
		\STATE {\textbf{Input:}} Observation matrix \(\{\mathbb{O}_a\}_{a \in \mathcal{A}}\), confidence level \(\delta\), length of initial episode \(T_0\)
		\STATE {\textbf{Initialize:}} \(t = 0\), \(k = 0\), policy \(\pi_0\) uniform on actions \(\mathcal{A}\), belief \(b_0\) uniform over states \(\mathcal{S}\), collected pairs of samples \(\mathcal{G} = \emptyset\)
		\WHILE{\(t < T\)}
		\IF{\(k>0\)}
		\STATE Compute transition model \(\widehat{\mathbb{T}}=\{\widehat{\mathbb{T}}_a\}_{a \in \mathcal{A}}\) from \(\mathcal{G}\) using Algorithm~\ref{alg:AOASEstimation}
		\STATE Build a confidence region \(\mathcal{C}_{a,k}(\delta_{a,k})\) around each \(\widehat{\mathbb{T}}_a\)
		\STATE Define the set \(\mathcal{C}_k(\delta_k)\) of admissible POMDPs 
		\STATE Get policy \(\pi_{k}\) from the oracle (Equation~\ref{eq:oracleOutput})
		\ENDIF
		\STATE Set \(n_k(a) = 0 \quad \forall a \in \mathcal{A}\)
		\STATE Execute \(a_t = \pi_k(b_t)\)
		\STATE Observe \(o_t\), get reward \(r_t = r(o_t)\)
		\STATE Update belief to \(b_{t+1}\) using Equation~\eqref{eq:beliefUpdate}
		\STATE Set \(t=t+1\)
		\WHILE{\(t < T_0\) or \(n_k(\pi_k(b_t)) < N_k(\pi_k(b_t))\)}\label{line:stoppingCondition}

		\STATE Execute \(a_{t} = \pi_k(b_{t})\)
		\STATE Observe \(o_{t}\), get reward \(r_{t} = r(o_{t})\)
		\STATE Update belief to \(b_{t+1}\) using Equation~\eqref{eq:beliefUpdate}
		\STATE Update \(n_k(a_{t-1}) = n_k(a_{t-1}) + 1\)
		\STATE Add \((a_{t-1}, a_t, o_{t-1}, o_t)\) to \(\mathcal{G}\)
		\STATE Set \(t = t+1\)
		\ENDWHILE
		\STATE Set \(N_{k+1}(a) = N_k(a) + n_k(a) \quad \forall a \in \mathcal{A}\)
		\STATE Set \(k = k+1\)
		\ENDWHILE
	\end{algorithmic}
\end{algorithm}

\section{ACTION-WISE OAS-UCRL ALGORITHM}\label{sec:OASUCRLAlgorithm}
In this section, we present the \emph{Action-wise} OAS-UCRL (AOAS-UCRL) algorithm which is inspired by the combination of an optimistic approach mimicking the UCRL algorithm~\citep{jaksch2010near} and the \emph{Action-wise} OAS estimation procedure. The algorithm starts with an initial episode \(k=0\) of length \(T_0\) where a uniform exploration policy is used to collect samples. At the beginning of each successive episode \(k\), all the samples collected up to that moment are provided to the \emph{Action-wise} OAS Algorithm. The \emph{Action-wise} OAS algorithm returns as output the estimated transition model \(\widehat{\mathbb{T}} = \{\widehat{\mathbb{T}}_a\}_{a \in \mathcal{A}}\). 
The algorithm then proceeds by building a confidence region \(\mathcal{C}_{a,k}(\delta_{a,k})\) around every \(\widehat{\mathbb{T}}_a\) such that the real transition matrix lies in it with high probability, namely \(P(\mathbb{T}_a \in \mathcal{C}_{a,k}(\delta_{a,k})) \ge 1 - \delta_{a,k}\), with \(\delta_{a,k} \coloneq \delta/(Ak^3)\). These confidence regions together define the confidence region \(\mathcal{C}_{k}(\delta_k)\) of the real POMDP instance \(\mathcal{Q}\), for which in turn it holds that \(P(\mathcal{Q} \in \mathcal{C}_k(\delta_k)) \ge 1 - \delta_k\), with \(\delta_k \coloneq \delta/k^3\).
As specified in Section~\ref{sec:problemFormulation}, we assume the existence of an oracle that is able to compute the optimal policy corresponding to the optimistic POMDP contained in \(\mathcal{C}_k(\delta_k)\). More formally, the oracle computes:
\begin{align}\label{eq:oracleOutput}
	\pi_k = \arg \underset{\pi \in \mathcal{P}}{\max} \;\; \underset{\widetilde{\mathcal{Q}} \in \mathcal{C}_k(\delta_k)}{\max}\;\; \rho(\pi, \widetilde{\mathcal{Q}}),
\end{align}
where we used \(\rho(\pi, \widetilde{\mathcal{Q}})\) to emphasize the dependence of the average reward from policy \(\pi\) and the POMDP instance \(\widetilde{Q}\).\\
The policy \(\pi_k\) returned by the oracle is then used during episode \(k\) to interact with the environment. An episode \(k\) terminates whenever there is at least an action \(a\) such that the number \(n_k(a)\) of times it appears as a first element in that episode matches the total number of times \(N_k(a)\) it appears as a first element from the beginning of the interaction (line~\ref{line:stoppingCondition}). The pseudocode of the approach is reported in Algorithm~\ref{alg:oasAlgorithm}.

\begin{figure*}[t]
\begin{minipage}{.50\textwidth}
    \vspace{-.6cm}
  \includegraphics[scale=1.]{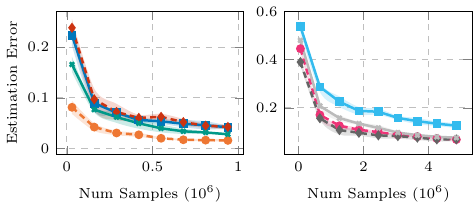}
  \captionsetup{width=0.92\textwidth}
  \caption{Error in Frobenious norm of the Different 
  Action Transition Matrices under two POMDP Instances (10 runs, 95 \%c.i.).}
  \label{fig:estimationError}
\end{minipage}%
\hspace{.55cm} 
\begin{minipage}{.4\textwidth}
	\vspace{-.25cm}
  \includegraphics[width=\textwidth]{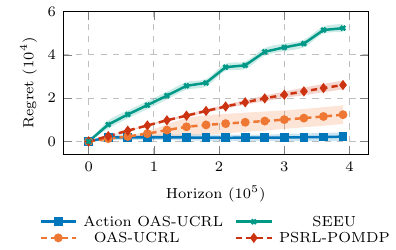}
  \captionsetup{width=1.1\textwidth}
  \caption{Regret Results on a POMDP Instance with \(S=3\), \(A=4\), \(O=4\) (10 runs, 95 \%c.i.).}
  \label{fig:regret}
\end{minipage}
\end{figure*}

We prove the following result for the \emph{Action-wise} OAS-UCRL algorithm. The proof is reported in Appendix~\ref{appendix:theoremProof}.
\begin{restatable}[]{theorem}{algorithmBound}\label{theorem:algorithmBound}
	Let us assume to have a POMDP instance \(\mathcal{Q}\) satisfying Assumptions~\ref{ass:minElem} and~\ref{ass:weaklyRev}. If the Action-wise OAS-UCRL algorithm is run for \(T\) steps, with probability at least \(1 - 2\delta\), it suffers from a total regret:
\begin{equation*}
	\mathcal{R}_T \le \mathcal{O}\left( \frac{CD\widetilde{G}}{\alpha^2 \widetilde{d}_{\min}}\sqrt{SA^3T\,\log T\,\log O}\right).
\end{equation*}
	where \(C\coloneq \frac{4(1-\epsilon)^3}{\epsilon^4}\) and \(D\) is a finite constant bounding the span of the bias function (definition in Proposition~\ref{prop:uniformBoundBias}).
\end{restatable}
This result is achieved also thanks to a new bound on the belief error presented in Lemma~\ref{lemma:BoundOnBeliefErrors}.\\
To the best of our knowledge, this is the first algorithm enjoying a regret of order \(\mathcal{O}(\sqrt{T \log T})\) in the average-reward POMDP setting when compared against the optimal policy, thus improving over state-of-the-art approaches. Indeed, the PSRL-POMDP algorithm~\citep{jahromi2022online} reaches a \(\mathcal{O}(T^{2/3})\) regret guarantee, but their result holds under the assumption of employing a consistent estimator, which they do not provide.\\
Concerning the OAS-UCRL approach~\citep{russo2024efficient}, it reaches a \(\mathcal{O}(\sqrt{T \log T})\) regret guarantee when compared to the weaker class of stochastic belief-based policies. It is possible to show that, by optimizing the minimum action probability \(\iota>0\) characterizing their policy class, their approach suffers a regret of order \(\widetilde{\mathcal{O}}(T^{4/5})\) when compared against the optimal policy.


\section{NUMERICAL SIMULATIONS}\label{sec:numericalSimulations}
In this section, we will provide numerical simulations testing both the AOAS estimation procedure and the AOAS-UCRL algorithm presented in the previous sections.
Further experiments and simulation details are reported in Appendix~\ref{appendix:simulationDetails}. 

\paragraph{Estimation Performance} This first set of experiments shows the estimation error of the transition model for two different POMDP instances when the \emph{Action-wise} OAS algorithm is employed. In particular, the objective is to show that AOAS is able to reduce the estimation error when using samples collected from different policies. On the left plot of Figure~\ref{fig:estimationError}, the POMDP has \(S=5\) states, \(A=4\) actions and \(O=8\) observations, while on the right the values are \(S=10\), \(A=4\) and \(O=16\). Each line in the plot represents the estimation error of the associated transition matrix \(\mathbb{T}_a\). Samples are collected using belief-based policies that change every \(10^4\) steps. The implemented policies play the action maximizing the immediate reward given the current belief and their change is determined by varying the transition model used to update the belief vector\footnote{We refer to the transition model adopted in Equation~\ref{eq:beliefUpdate} to update the belief, which can be arbitrarily different from the real transition model.}. To promote the choice of different actions, we employ stochastic policies. We notice that the approach works well also in larger problem instances, such as the one on the right. From details reported in Appendix~\ref{appendix:simulationDetails} about the considered instances, we observe that the actions having higher error are either those that have been chosen less (so fewer samples are available) or those associated with a block diagonal matrix \(\mathbb{B}_a\) with low values of \(\sigma_{\min}(\mathbb{B}_a)\).

\paragraph{Regret Results} This second set of experiments compares the \emph{Action-wise} OAS-UCRL algorithm with different baseline approaches. We exclude from the comparison the SM-UCRL algorithm from~\citet{Azizzadenesheli2016Reinforcement} since it employs memoryless policies which are known to yield a linear regret with respect to our oracle. 
We compare against the SEEU~\citep{xiong2022sublinear} approach but we use a modified version that provides the algorithm with information about the observation model. However, this approach is the one showing the highest regret under the considered instance and this aspect is mainly due to (i) the need for SD approaches of a large number of samples to provide good estimates and (ii) the inherent nature of the algorithm which alternates between purely exploratory and purely exploitative phases.\\
The comparison proceeds with other baseline algorithms naively developed under the assumption of knowing the observation model. It is possible to see that our solution outperforms the competing alternatives, thus validating the theoretical results. 
Concerning the PSRL-POMDP algorithm, we chose to implement it using a standard particle filter approach which however lacks estimation guarantees, as reflected in the suffered regret. Regarding the OAS-UCRL algorithm instead, the higher regret with respect to our algorithm can be attributed to the stochasticity of the employed policy (we set the minimum action probability to \(\iota = 0.025\)) but also to the less accurate estimates of the model parameters since OAS-UCRL only uses samples coming from the last episode for model estimation, thus discarding all previous ones.

In Appendix~\ref{appendix:simulationDetails}, we present two different ablation studies. (i) The first shows the regret performance of Action OAS-UCRL when compared against the OAS-UCRL algorithm run under different values of the minimum action probability \(\iota\), (ii) the second one instead explores the impact of the samples reuse strategy adopted by Action OAS-UCRL. 

\section{CONCLUSIONS AND FUTURE WORKS}\label{sec:conclusions}
We introduced a novel estimation procedure to learn the POMDP parameters assuming the knowledge of the observation model. We showed how this approach can be used to learn the transition matrix associated with each action separately; we proved consistency for it, and we highlighted that it can even be used with samples collected under different policies. After that, we proposed the \emph{Action-wise} OAS-UCRL algorithm which, to the best of our knowledge, is the first to achieve a regret guarantee of order \(\widetilde{\mathcal{O}}(\sqrt{T})\) when compared against the optimal policy in the average reward POMDP setting. We reached this result thanks to the new proposed estimation procedure and new tighter theoretical results on the estimated belief error.\\
In future work, we plan to extend these techniques to a more standard setting where neither the transition nor the observation model is available.

\subsubsection*{Acknowledgements}
This paper is supported by FAIR (Future Artificial Intelligence Research) project, funded by the NextGenerationEU program within the PNRR-PE-AI scheme (M4C2, Investment 1.3, Line on Artificial Intelligence).

\bibliographystyle{plainnat}
\bibliography{references}

\begin{thebibliography}{37}
\providecommand{\natexlab}[1]{#1}
\providecommand{\url}[1]{\texttt{#1}}
\expandafter\ifx\csname urlstyle\endcsname\relax
  \providecommand{\doi}[1]{doi: #1}\else
  \providecommand{\doi}{doi: \begingroup \urlstyle{rm}\Url}\fi

\bibitem[Anandkumar et~al.(2014)Anandkumar, Ge, Hsu, Kakade, and
  Telgarsky]{Anandkumar2014Tensor}
Animashree Anandkumar, Rong Ge, Daniel Hsu, Sham~M. Kakade, and Matus
  Telgarsky.
\newblock Tensor decompositions for learning latent variable models.
\newblock \emph{J. Mach. Learn. Res.}, 15\penalty0 (1):\penalty0 2773–2832,
  jan 2014.

\bibitem[Azar et~al.(2017)Azar, Osband, and Munos]{azar2017MinimaxRB}
Mohammad~Gheshlaghi Azar, Ian Osband, and R{\'e}mi Munos.
\newblock Minimax regret bounds for reinforcement learning.
\newblock In \emph{International Conference on Machine Learning}, 2017.

\bibitem[Azizzadenesheli et~al.(2016)Azizzadenesheli, Lazaric, and
  Anandkumar]{Azizzadenesheli2016Reinforcement}
Kamyar Azizzadenesheli, Alessandro Lazaric, and Anima Anandkumar.
\newblock Reinforcement learning of pomdps using spectral methods.
\newblock In \emph{Annual Conference Computational Learning Theory}, 2016.

\bibitem[Azuma(1967)]{azuma1967weighted}
Kazuoki Azuma.
\newblock {Weighted sums of certain dependent random variables}.
\newblock \emph{Tohoku Mathematical Journal}, 1967.

\bibitem[Bartlett and Tewari(2009)]{bartlett2009regal}
Peter~L. Bartlett and Ambuj Tewari.
\newblock Regal: a regularization based algorithm for reinforcement learning in
  weakly communicating mdps.
\newblock 2009.

\bibitem[Bertsekas(1995)]{bertsekas1995dynamic}
Dimitri Bertsekas.
\newblock \emph{Dynamic Programming and Optimal Control}, volume~1.
\newblock 01 1995.

\bibitem[Chen et~al.(2023)Chen, Wang, Xiong, Mei, and Bai]{chen2023lower}
Fan Chen, Huan Wang, Caiming Xiong, Song Mei, and Yu~Bai.
\newblock Lower bounds for learning in revealing pomdps.
\newblock In \emph{Proceedings of the 40th International Conference on Machine
  Learning}, ICML'23. JMLR.org, 2023.

\bibitem[De~Castro et~al.(2017)De~Castro, Gassiat, and
  Le~Corff]{deCastroConsistent2017}
Yohann De~Castro, Elisabeth Gassiat, and Sylvain Le~Corff.
\newblock Consistent estimation of the filtering and marginal smoothing
  distributions in nonparametric hidden markov models.
\newblock \emph{IEEE Transactions on Information Theory}, 63\penalty0
  (8):\penalty0 4758--4777, 2017.
\newblock \doi{10.1109/TIT.2017.2696959}.

\bibitem[Guo et~al.(2016)Guo, Doroudi, and Brunskill]{guo2016pac}
Zhaohan Guo, Shayan Doroudi, and Emma Brunskill.
\newblock A pac rl algorithm for episodic pomdps.
\newblock 05 2016.

\bibitem[Hauskrecht and Fraser(2000)]{hauskrecht2000PlanningTO}
Milos Hauskrecht and Hamish S.~F. Fraser.
\newblock Planning treatment of ischemic heart disease with partially
  observable markov decision processes.
\newblock \emph{Artificial intelligence in medicine}, 2000.

\bibitem[Hsu et~al.(2012)Hsu, Kakade, and Zhang]{hsu2012spectral}
Daniel Hsu, Sham~M. Kakade, and Tong Zhang.
\newblock A spectral algorithm for learning hidden markov models.
\newblock \emph{Journal of Computer and System Sciences}, 78\penalty0
  (5):\penalty0 1460--1480, 2012.
\newblock ISSN 0022-0000.
\newblock JCSS Special Issue: Cloud Computing 2011.

\bibitem[Jafarnia~Jahromi et~al.(2022)Jafarnia~Jahromi, Jain, and
  Nayyar]{jahromi2022online}
Mehdi Jafarnia~Jahromi, Rahul Jain, and Ashutosh Nayyar.
\newblock Online learning for unknown partially observable mdps.
\newblock In \emph{Proceedings of The 25th International Conference on
  Artificial Intelligence and Statistics}, 2022.

\bibitem[Jaksch et~al.(2010)Jaksch, Ortner, and Auer]{jaksch2010near}
Thomas Jaksch, Ronald Ortner, and Peter Auer.
\newblock Near-optimal regret bounds for reinforcement learning.
\newblock \emph{J. Mach. Learn. Res.}, 11:\penalty0 1563–1600, aug 2010.
\newblock ISSN 1532-4435.

\bibitem[Jiang et~al.(2023)Jiang, Jiang, Li, Lin, Wang, and
  Zhou]{jiang2023online}
Bowen Jiang, Bo~Jiang, Jian Li, Tao Lin, Xinbing Wang, and Chenghu Zhou.
\newblock Online restless bandits with unobserved states.
\newblock In \emph{Proceedings of the 40th International Conference on Machine
  Learning}, 2023.

\bibitem[Jin et~al.(2018)Jin, Allen-Zhu, Bubeck, and Jordan]{jin2018qlearning}
Chi Jin, Zeyuan Allen-Zhu, Sebastien Bubeck, and Michael~I Jordan.
\newblock Is q-learning provably efficient?
\newblock In S.~Bengio, H.~Wallach, H.~Larochelle, K.~Grauman, N.~Cesa-Bianchi,
  and R.~Garnett, editors, \emph{Advances in Neural Information Processing
  Systems}, volume~31. Curran Associates, Inc., 2018.

\bibitem[Jin et~al.(2020)Jin, Kakade, Krishnamurthy, and Liu]{jin2020sample}
Chi Jin, Sham~M. Kakade, Akshay Krishnamurthy, and Qinghua Liu.
\newblock Sample-efficient reinforcement learning of undercomplete pomdps.
\newblock In \emph{Proceedings of the 34th International Conference on Neural
  Information Processing Systems}, NIPS'20, Red Hook, NY, USA, 2020. Curran
  Associates Inc.
\newblock ISBN 9781713829546.

\bibitem[Krishnamurthy(2016)]{krish2016partially}
Vikram Krishnamurthy.
\newblock \emph{Partially Observed Markov Decision Processes: From Filtering to
  Controlled Sensing}.
\newblock Cambridge University Press, 2016.
\newblock \doi{10.1017/CBO9781316471104}.

\bibitem[Liu et~al.(2022{\natexlab{a}})Liu, Chung, Szepesvári, and
  Jin]{liu2022partially}
Qinghua Liu, Alan Chung, Csaba Szepesvári, and Chi Jin.
\newblock When is partially observable reinforcement learning not scary?,
  2022{\natexlab{a}}.

\bibitem[Liu et~al.(2022{\natexlab{b}})Liu, Netrapalli, Szepesvári, and
  Jin]{liu2022optimistic}
Qinghua Liu, Praneeth Netrapalli, Csaba Szepesvári, and Chi Jin.
\newblock Optimistic mle -- a generic model-based algorithm for partially
  observable sequential decision making, 2022{\natexlab{b}}.

\bibitem[Madani(1999)]{madani1999computability}
Omid Madani.
\newblock On the computability of infinite-horizon partially observable markov
  decision processes.
\newblock 12 1999.

\bibitem[Mahadevan(1996)]{Mahadevan1996average}
Sridhar Mahadevan.
\newblock Average reward reinforcement learning: Foundations, algorithms, and
  empirical results.
\newblock 1996.

\bibitem[Maillard and Mannor(2014)]{MaillardLatent2014}
Odalric-Ambrym Maillard and Shie Mannor.
\newblock Latent bandits.
\newblock \emph{31st International Conference on Machine Learning, ICML 2014},
  1, 05 2014.

\bibitem[Marco et~al.(2017)Marco, Berkenkamp, Hennig, Schoellig, Krause,
  Schaal, and Trimpe]{alonso2017virtual}
Alonso Marco, Felix Berkenkamp, Philipp Hennig, Angela~P. Schoellig, Andreas
  Krause, Stefan Schaal, and Sebastian Trimpe.
\newblock Virtual vs. real: Trading off simulations and physical experiments in
  reinforcement learning with bayesian optimization.
\newblock \emph{CoRR}, 2017.

\bibitem[Mossel and Roch(2005)]{mossel2005learning}
Elchanan Mossel and S\'{e}bastien Roch.
\newblock Learning nonsingular phylogenies and hidden markov models.
\newblock Association for Computing Machinery, 2005.

\bibitem[Ortner and Ryabko(2012)]{ortner2012e}
Ronald Ortner and Daniil Ryabko.
\newblock Online regret bounds for undiscounted continuous reinforcement
  learning.
\newblock In \emph{Neural Information Processing Systems}, 2012.

\bibitem[Png et~al.(2012)Png, Pineau, and Chaib-Draa]{png2012building}
Shaowei Png, Joelle Pineau, and Brahim Chaib-Draa.
\newblock Building adaptive dialogue systems via bayes-adaptive pomdps.
\newblock \emph{IEEE Journal of Selected Topics in Signal Processing}, 2012.

\bibitem[Ramponi et~al.(2020)Ramponi, Likmeta, Metelli, Tirinzoni, and
  Restelli]{ramponiTruly2020}
Giorgia Ramponi, Amarildo Likmeta, Alberto~Maria Metelli, Andrea Tirinzoni, and
  Marcello Restelli.
\newblock Truly batch model-free inverse reinforcement learning about multiple
  intentions.
\newblock In \emph{Proceedings of the Twenty Third International Conference on
  Artificial Intelligence and Statistics}, Proceedings of Machine Learning
  Research, 2020.

\bibitem[Russo et~al.(2024{\natexlab{a}})Russo, Metelli, and
  Restelli]{russo2024efficient}
Alessio Russo, Alberto~Maria Metelli, and Marcello Restelli.
\newblock Efficient learning of pomdps with known observation model in
  average-reward setting, 2024{\natexlab{a}}.
\newblock URL \url{https://arxiv.org/abs/2410.01331}.

\bibitem[Russo et~al.(2024{\natexlab{b}})Russo, Metelli, and
  Restelli]{russo2024switching}
Alessio Russo, Alberto~Maria Metelli, and Marcello Restelli.
\newblock Switching latent bandits.
\newblock \emph{Transactions on Machine Learning Research}, 2024{\natexlab{b}}.

\bibitem[Sondik(1978)]{sondik1978optimal}
Edward~J. Sondik.
\newblock The optimal control of partially observable markov processes over the
  infinite horizon: Discounted costs.
\newblock \emph{Operations Research}, 26\penalty0 (2):\penalty0 282--304, 1978.

\bibitem[Sutton and Barto(2018)]{sutton1998}
Richard~S. Sutton and Andrew~G. Barto.
\newblock \emph{Reinforcement Learning: An Introduction}.
\newblock The MIT Press, second edition, 2018.

\bibitem[Thananjeyan et~al.(2021)Thananjeyan, Kandasamy, Stoica, Jordan,
  Goldberg, and Gonzalez]{thanan2021resource}
Brijen Thananjeyan, Kirthevasan Kandasamy, Ion Stoica, Michael Jordan, Ken
  Goldberg, and Joseph Gonzalez.
\newblock Resource allocation in multi-armed bandit exploration: Overcoming
  sublinear scaling with adaptive parallelism.
\newblock PMLR, 2021.

\bibitem[Tropp(2010)]{tropp2010user}
Joel~A. Tropp.
\newblock User-friendly tail bounds for sums of random matrices.
\newblock \emph{Foundations of Computational Mathematics}, 12:\penalty0
  389--434, 2010.

\bibitem[Xiong et~al.(2022)Xiong, Chen, Gao, and Zhou]{xiong2022sublinear}
Yi~Xiong, Ningyuan Chen, Xuefeng Gao, and Xiang Zhou.
\newblock Sublinear regret for learning pomdps, 2022.

\bibitem[Zanette and Brunskill(2019)]{zanette2019tighter}
Andrea Zanette and Emma Brunskill.
\newblock Tighter problem-dependent regret bounds in reinforcement learning
  without domain knowledge using value function bounds.
\newblock In Kamalika Chaudhuri and Ruslan Salakhutdinov, editors,
  \emph{Proceedings of the 36th International Conference on Machine Learning},
  volume~97 of \emph{Proceedings of Machine Learning Research}, pages
  7304--7312. PMLR, 09--15 Jun 2019.

\bibitem[Zhou et~al.(2021)Zhou, Xiong, Chen, and Gao]{zhou2021regime}
Xiang Zhou, Yi~Xiong, Ningyuan Chen, and Xuefeng Gao.
\newblock Regime switching bandits.
\newblock In A.~Beygelzimer, Y.~Dauphin, P.~Liang, and J.~Wortman Vaughan,
  editors, \emph{Advances in Neural Information Processing Systems}, 2021.

\bibitem[Åström(1965)]{astrom1965optimal}
K.J Åström.
\newblock Optimal control of markov processes with incomplete state
  information.
\newblock \emph{Journal of Mathematical Analysis and Applications}, 1965.

\end{thebibliography}

\section*{Checklist}

 \begin{enumerate}

 \item For all models and algorithms presented, check if you include:
 \begin{enumerate}
   \item A clear description of the mathematical setting, assumptions, algorithm, and/or model. \textbf{Yes}
   \item An analysis of the properties and complexity (time, space, sample size) of any algorithm. \textbf{Yes}, we focused on the properties and highlighted the dependency with respect to the sample size.
   \item (Optional) Anonymized source code, with specification of all dependencies, including external libraries. \textbf{No}
 \end{enumerate}

 \item For any theoretical claim, check if you include:
 \begin{enumerate}
   \item Statements of the full set of assumptions of all theoretical results. \textbf{Yes}
   \item Complete proofs of all theoretical results. \textbf{Yes}
   \item Clear explanations of any assumptions. \textbf{Yes}     
 \end{enumerate}

 \item For all figures and tables that present empirical results, check if you include:
 \begin{enumerate}
   \item The code, data, and instructions needed to reproduce the main experimental results (either in the supplemental material or as a URL). \textbf{Yes/No}, instructions and parameters are provided in Appendix~\ref{appendix:simulationDetails}.
   \item All the training details (e.g., data splits, hyperparameters, how they were chosen). \textbf{Yes}, see Appendix~\ref{appendix:simulationDetails}.
    \item A clear definition of the specific measure or statistics and error bars (e.g., with respect to the random seed after running experiments multiple times). \textbf{Yes}, they are reported in the captions of the figures.
    \item A description of the computing infrastructure used. (e.g., type of GPUs, internal cluster, or cloud provider). \textbf{Yes}
 \end{enumerate}

 \item If you are using existing assets (e.g., code, data, models) or curating/releasing new assets, check if you include:
 \begin{enumerate}
   \item Citations of the creator If your work uses existing assets. \textbf{Not Applicable}
   \item The license information of the assets, if applicable. \textbf{Not Applicable}
   \item New assets either in the supplemental material or as a URL, if applicable. \textbf{Not Applicable}
   \item Information about consent from data providers/curators. \textbf{Not Applicable}
   \item Discussion of sensible content if applicable, e.g., personally identifiable information or offensive content. \textbf{Not Applicable}
 \end{enumerate}

 \item If you used crowdsourcing or conducted research with human subjects, check if you include:
 \begin{enumerate}
   \item The full text of instructions given to participants and screenshots. \textbf{Not Applicable}
   \item Descriptions of potential participant risks, with links to Institutional Review Board (IRB) approvals if applicable. \textbf{Not Applicable}
   \item The estimated hourly wage paid to participants and the total amount spent on participant compensation. \textbf{Not Applicable}
 \end{enumerate}

 \end{enumerate}

\onecolumn

\renewcommand\thesection{\Alph{section}}
\setcounter{section}{0} 

\section*{APPENDIX ORGANIZATION}
Here is an outline of the appendix.
\begin{itemize}
	\item Section~\ref{appendix:comparison} presents a comparison of our work with the most relevant related works.
	\item Section~\ref{appendix:lemmaProof} and Section~\ref{appendix:theoremProof} are devoted respectively to the proof of Lemma~\ref{lemma:combinedActionDistr} and Theorem~\ref{theorem:algorithmBound}.
	\item Section~\ref{appendix:beliefNewConcentration} contains Lemma~\ref{lemma:BoundOnBeliefErrors}, a result that relates the error in the belief vector with the error in the estimated transition model. This bound improves over existing results and is crucial for proving Theorem~\ref{theorem:algorithmBound}.
	\item Section~\ref{appendix:usefulLemmas} is miscellaneous of new and existing useful results that support the theoretical analysis of the work.
	\item Section~\ref{appendix:simulationDetails} provides further details on the simulations and their results presented in the main paper. 
\end{itemize}

\vspace{0.5cm}

\section{COMPARISON WITH RELATED WORKS}\label{appendix:comparison}
This section is devoted to a more detailed comparison of our work with state-of-the-art approaches in this setting. In particular, we confront with: the SM-UCRL algorithm~\citep{Azizzadenesheli2016Reinforcement}, the SEEU algorithm~\citep{xiong2022sublinear}, the PSRL-POMDP algorithm~\citep{jahromi2022online} and the OAS-UCRL algorithm~\citep{russo2024efficient} (Section~\ref{appendix:comparisonOAS}).

\begin{itemize}
	\item the \textbf{SM-UCRL algorithm}~\citep{Azizzadenesheli2016Reinforcement} tackles the  POMDP learning problem without assuming knowledge of either the transition or the observation model. It employs standard Spectral Decomposition approaches to learn the model parameters. However, it assumes the class of memoryless policies that are characterized by choosing the next action only conditioning on the last observation seen: these policies are suboptimal in the POMDP setting. Furthermore, they assume that each action is always chosen with a minimum probability \(\iota > 0\). The SM-UCRL algorithm works in different episodes and the model parameters are computed at the beginning of each episode using only samples collected during the previous episode, thus discarding all the others. The algorithm reaches a \(\widetilde{\mathcal{O}}(\sqrt{T})\) regret when compared against the Optimal Stochastic Memoryless policy.
	\item the \textbf{SEEU algorithm}~\citep{xiong2022sublinear} considers again the standard POMDP setting without having partial knowledge of the model. The algorithm alternates between purely exploratory phases when collected samples are used to estimate the model parameters using Spectral approaches, and exploitative phases where an optimal policy is used on the learned optimistic POMDP. This algorithm shows a \(\widetilde{\mathcal{O}}(T^{2/3})\) regret guarantee when compared against the optimal Belief-based policy, which is the optimal class of policy in this setting. Since they consider the class of belief-based policies, they require, as in our case, the one-step reachability assumption~\ref{ass:minElem} to obtain regret guarantees for their approach.
	\item the \textbf{PSRL-POMDP algorithm}~\citep{jahromi2022online} considers a POMDP setting with a known observation model but an unknown transition model. They employ a Bayesian approach to update the model parameters at each timestamp. However, they do not provide a consistent approach for model estimation and base their results on the assumption that the employed estimates are consistent, thus obtaining a more accurate estimate as more samples are acquired. Under this assumption, and an analogous assumption on the consistency of the belief estimates, they show a Bayesian regret of order \(\mathcal{O}(T^{2/3})\) against the optimal POMDP policy. 
\end{itemize}

\subsection{Comparison between the \emph{Action-wise} OAS-UCRL Algorithm and the OAS-UCRL Algorithm of~\citep{russo2024efficient}}\label{appendix:comparisonOAS}
In this section, we will show the main differences of our \emph{Action-wise} OAS-UCRL with respect to the OAS-UCRL approach described in~\citep{russo2024efficient}. In particular, we highlight that:
\begin{itemize}
	\item the OAS-UCRL approach employs the class of stochastic belief-based policies for which each action has a minimum probability \(\iota > 0\) of being chosen at each timestamp. This ensures a continual refinement of the estimate of the transition matrix \(\mathbb{T}_a\) over time. 
	\item Because of the previous point, in their regret analysis, they compare against the optimal stochastic belief-based policy and they reach a regret guarantee of order \(\widetilde{\mathcal{O}}(\sqrt{T})\) with respect to this oracle. However, we improve over their result since we obtain a \(\widetilde{\mathcal{O}}(\sqrt{T})\) regret guarantee when compared with the optimal deterministic belief-based policy.\\
	It is indeed possible to show that, by optimizing their regret result over the minimum action probability \(\iota\), the OAS-UCRL algorithm suffers regret \(\widetilde{\mathcal{O}}(T^{4/5})\)  when compared against the optimal policy. We observe that the regret of the OAS-UCRL approach with respect to the optimal stochastic policy can be bounded by\footnote{Here, we disregard logarithmic terms.} \(C \sqrt{T}/\iota^{3/2}\), where \(C\) is a constant related to the problem parameters. When compared against the optimal POMDP policy, the regret of the OAS-UCRL algorithm can be expressed as:
    \begin{align}\label{eq:OASUCRLAlgorithmRegret}
        \mathcal{R}_T \le T\, (A-1) \iota + C \frac{\sqrt{T}}{\iota^{3/2}} 
    \end{align}
    where we introduced an additional term \(T\, (A-1) \iota\) representing the regret suffered when choosing the suboptimal action, which happens with probability \((A-1) \iota\).    
    This probability is then multiplied by the total interaction horizon \(T\).\\ 
    By optimizing the regret in~\eqref{eq:OASUCRLAlgorithmRegret} with respect to the minimum action probability \(\iota\), we obtain a final regret order of \(\widetilde{\mathcal{O}}(T^{4/5})\) for the OAS-UCRL algorithm under the optimal POMDP policy.
\end{itemize}
The improvements in terms of regret of the \emph{Action-wise}  OAS-UCRL algorithm over the OAS-UCRL counterpart are mainly due to: (i) a tighter analysis of the belief estimation error (see Lemma~\ref{lemma:BoundOnBeliefErrors}); (ii) the differences in the employed estimation procedure.\\
We report here the main differences between the \emph{Action-wise} OAS and the OAS procedures:
\begin{itemize}
	\item The OAS procedure focuses on estimating the stationary distribution on action-observation pairs\footnote{They estimate \(d_{A^2O^2}^{\pi}(a,a',o, o'):= \lim_{t \to \infty}d_t^\pi(a,a',o,o')\) with \(d_t^\pi(a,a',o,o'):= P(A_t=a, A_{t+1}=a', O_t=o, O_{t+1}=o'|\pi)\).} induced by the employed policy \(\pi\) while the \emph{Action-wise} OAS procedure estimates distributions defined over a finite amount of samples and conditioned on the event defined in Equation~\eqref{def:conditionalEvent}. The distribution we consider allows us to obtain estimation guarantees separately for each action transition matrix \(\mathbb{T}_a\), while the OAS procedure provides estimation guarantees for the whole transition model, namely it bounds \(\sum_{a \in \mathcal{A}}\|\mathbb{T}_a - \widehat{\mathbb{T}}_a\|_F\). Indeed, estimating the stationary distribution conditioned on action \(a\) was not a viable choice since, by removing the minimum action probability assumption used in~\citet{russo2024switching}, there may be cases where this conditional stationary distribution would not exist.\\
    This happens for example when a given action \(a\) is not played under stationary conditions, namely \(\lim_{t \to \infty}P(A_t=a)=0\), and this prevents us from defining the following stationary conditional distribution:
    \begin{align*}
        d^{(a)}_{AO^2}(a',o,o')\coloneq \lim_{t \to \infty}d_t^{\pi}(a',o,o'|a),        
    \end{align*} 
    with \(d_t^{\pi}(a',o,o'|a)\coloneq P(A_{t+1}=a', O_t=o, O_{t+1}=o'| A_t=a)\). This aspect becomes crucial when the chain does not start from stationarity. Indeed, in this case, we may get some samples from \(a\) if, for some initial \(t\), we have \(P(A_t=a) > 0\) but we would not be able to use them to define an estimate of \(d^{(a)}_{AO^2}\) since this conditional distribution does not exist.\\
    The scenario detailed above led us to opt for a distribution different from the stationary one.
	\item We show that the \emph{Action-wise} OAS estimation procedure works also for samples collected under different policies. This improves the sample efficiency of the approach with respect to the OAS method which is instead characterized by only employing samples deriving from a unique distribution.
	\item As a drawback, our \emph{Action-wise} OAS procedure requires Assumption~\ref{ass:minElem} to hold, while the OAS approach only requires the ergodicity of the induced chain, which is a weaker condition than Assumption~\ref{ass:minElem}.
\end{itemize}

\begin{table}[h!]
\centering
\caption{Table Comparing with the Most Relevant Related Works.}
\begin{tabularx}{\textwidth}{|>{\centering\arraybackslash}p{3.2cm}|>{\centering\arraybackslash}p{2.25cm}|>{\centering\arraybackslash}p{2.25cm}|
>{\centering\arraybackslash}p{2.25cm}|>{\centering\arraybackslash}p{2.25cm}|>{\columncolor{poliblue3!20}\centering\arraybackslash}p{2.32cm}|}
\hline
   \parbox[c][1.3cm]{3.5cm}{} & \textbf{SM-UCRL} & \textbf{SEEU} & 
 \textbf{PSRL-POMDP} & 
 \textbf{OAS-UCRL} & 
 \textbf{Action-wise OAS-UCRL}\\ \cline{1-6}
\end{tabularx}

\vspace{0.1cm}

\begin{tabularx}{\textwidth}{|>{\centering\arraybackslash}p{3.2cm}|>{\centering\arraybackslash}p{2.25cm}|>{\centering\arraybackslash}p{2.25cm}|
>{\centering\arraybackslash}p{2.25cm}|>{\centering\arraybackslash}p{2.25cm}|>{\columncolor{poliblue3!20}\centering\arraybackslash}p{2.32cm}|}
\hline
  
  \parbox[c][1.0cm]{3.5cm}{\centering
 {Knowledge of Observation Model}} & No & No & Yes & Yes & Yes \\ 

  \parbox[c][1.0cm]{3.3cm}{\centering
 {Ergodicity of Induced Chain}} & Yes & Yes & Yes & Yes & Yes \\  

  \parbox[c][1.0cm]{3.3cm}{\centering
  {Minimum Transition Probability}} & No & Yes & No & Yes & Yes \\ 

  \parbox[c][1.0cm]{3.3cm}{\centering
 {Invertible Transition Model}} & Yes & Yes & No & No & No \\  

  \parbox[c][1.0cm]{3.3cm}{\centering
  {Full-rank Observation Model*}} & Yes & Yes & Yes** & Yes & Yes \\ 

  \parbox[c][1.0cm]{3.3cm}{\centering
  {Minimum Action Probability}} & Yes & No & No & Yes & No \\ 

  \parbox[c][1.0cm]{3.3cm}{\centering
 {Consistent Transition Model Estimation}} & No & No & Yes & No & No \\  

  \parbox[c][1.0cm]{3.3cm}{\centering
 {Consistent Belief Estimation}} & No & No & Yes & No & No \\  \cline{1-6}

\end{tabularx}

\vspace{0.1cm}

\begin{tabularx}{\textwidth}{|>{\centering\arraybackslash}p{3.2cm}|>{\centering\arraybackslash}p{2.25cm}|>{\centering\arraybackslash}p{2.25cm}|
>{\centering\arraybackslash}p{2.25cm}|>{\centering\arraybackslash}p{2.25cm}|>{\columncolor{poliblue3!20}\centering\arraybackslash}p{2.32cm}|}
\hline
  
 \parbox[c][1.2cm]{3.5cm}{\centering
 {\textbf{Estimation Technique}}} & Spectral Decomp. & Spectral Decomp. & Bayesian Update & OAS & Action-wise OAS\\  \cline{1-6}
     
 \parbox[c][1.2cm]{3.5cm}{\centering
 {\textbf{Consistent Estimation}}} & \Large{\cmark} & \Large{\cmark} & \Large{\xmark} & \Large{\cmark} & \Large{\cmark} \\ \cline{1-5}

 \parbox[c][1.0cm]{3.5cm}{\centering
 {Handles Uniform Policies}} & \Large{\cmark} & \Large{\cmark} & \Large{\xmark} & \Large{\cmark} & \Large{\cmark} \\ 

 \parbox[c][1.0cm]{3.5cm}{\centering
 {Handles Memoryless Policies}} & \Large{\cmark} & \Large{\xmark} & \Large{\xmark} & \Large{\cmark} & \Large{\cmark}\\ 

 \parbox[c][1.0cm]{3.5cm}{\centering
 {Handles Belief-based Policies}} & \Large{\xmark} & \Large{\xmark} & \Large{\xmark} & \Large{\cmark} & \Large{\cmark}\\ \cline{1-6}
\end{tabularx}

\vspace{0.1cm}

\begin{tabularx}{\textwidth}{|>{\centering\arraybackslash}p{3.2cm}|>{\centering\arraybackslash}p{2.25cm}|>{\centering\arraybackslash}p{2.25cm}|
>{\centering\arraybackslash}p{2.25cm}|>{\centering\arraybackslash}p{2.25cm}|>{\columncolor{poliblue3!20}\centering\arraybackslash}p{2.32cm}|}
\hline

  \parbox[c][1.2cm]{3.5cm}{\centering
 {\textbf{Algorithm Type}}} & Optimistic & Alternating Explor-Optimistic & Bayesian & Optimistic & Optimistic\\ \cline{1-6}
    
  \parbox[c][1.2cm]{3.5cm}{\centering
 {\textbf{Oracle Policy}}} & Opt. Stochastic Memoryless & Opt. POMDP & Opt. POMDP & Opt. Stochastic POMDP & Opt. POMDP\\   \cline{1-6}

  \parbox[c][1.2cm]{3.5cm}{\centering
 {\textbf{Regret}}} & $\widetilde{\mathcal{O}}(\sqrt{T})$ & $\widetilde{\mathcal{O}}(T^{2/3})$ & $\mathcal{O}(T^{2/3})$ & $\widetilde{\mathcal{O}}(\sqrt{T})$ & $\widetilde{\mathcal{O}}(\sqrt{T})$\\  \cline{1-6}
\end{tabularx}
\label{tab:comparisonRL}
\end{table}

\subsection{Comparison Table}
Inspired by Table 1 in~\citet{russo2024efficient}, we present a similar comparison defined in terms of (i) required Assumptions (first sub-table), (ii) properties of used estimation techniques (second sub-table), and (iii) properties of the associated regret-minimization algorithm (third sub-table).\\
Concerning the first sub-table, the cells with \emph{Yes} denote a required assumption, and viceversa.\\
Some notes referring to the content of the table:
\begin{itemize}
	\item[] (*): Saying that a matrix is full-rank corresponds to saying, under the weakly-revealing terminology, that \(\alpha > 0\);
	\item[] (**): For the PSRL-POMDP algorithm, the full-rank observability assumption is reported in terms of the Kullback-Leibler divergence between the probability distributions on the next observation when conditioned on different transition models (see their Assumption 2 for details).
\end{itemize}

\vspace{0.5cm}

\section{PROOF OF LEMMA~\ref{lemma:combinedActionDistr}}\label{appendix:lemmaProof}

In this section, we will provide the proof for Lemma~\ref{lemma:combinedActionDistr}. This Lemma presents a bound on the estimation error of the action transition matrix \(\mathbb{T}_a\) when samples used to make the estimate come from different policies. First of all, we start by reporting the statement.

\combinedActionDistr*

\begin{proof}
This proof will be developed in two main parts: (i) the first one is devoted to showing the estimation error of the transition matrix both in the case of samples coming from a unique policy and the case of samples deriving from different policies; (ii) the second one shows how to reach theoretical guarantees on each action transition matrix \(\mathbb{T}_a\) starting from guarantees on the action-observation distribution derived in part (i).

\subsection*{First Part of the Proof} 
Let us now focus on the first part of the proof and let us first consider the case of samples obtained from a unique policy. 

Let us assume that a policy \(\pi \in \mathcal{P}\) is employed to interact with the environment for \(n+1\) steps and a dataset \(\mathcal{D}=\{(a_t,o_t)_{t=0}^{n}\}\) is generated. By grouping pairs of samples collected in consecutive timestamps, we define the new dataset \(\mathcal{G}= \{ (a_{t},a_{t+1},o_{t},o_{t+1})_{t=0}^{n-1}\}\) with cardinality \(n=|\mathcal{G}|\).

We recall here the definition of the distribution reported in Equation~\eqref{def:conditionalDist}, that is:
\begin{align*}
	\bm{d}^{(a,n, m)}_{AO^2} = \mathbb{E}_{\pi, \bm{\nu}}\Bigg[\frac{1}{m} \sum_{t=0}^{n-1} \mathds{1}\{a_{t}=a\}\; \bm{x}_t\ \Big| \mathcal{E}(a,n,m)\Bigg],
\end{align*}
where the event \(\mathcal{E}(a,n,m)\) holds true when the considered dataset \(\mathcal{G}\) has size \(n\) and the number of tuples in \(\mathcal{G}\) having action \(a\) as a first element coincides with \(m\).\\
For completeness, let us also provide the formal definition of the distribution \(\bm{d}^{(a,n, m)}_{AS^2} \in \Delta(\mathcal{A} \times \mathcal{S}^2)\) introduced in the main paper. Starting from the process that generated the dataset \(\mathcal{G}\) of consecutive pairs, let us assume to have access to the underlying states \((s_t)_{t=0}^n\). As done for \(\mathcal{G}\), we can define a dataset \(\mathcal{M}=\{(a_t,a_{t+1}, s_t, s_{t+1})_{t=0}^{n-1}\}\) having cardinality \(|\mathcal{M}|=n\). Let us also denote with \(\bm{y}_t \in \mathbb{R}^{AS^2}\) the one-hot encoded vector defined over the last three elements \((a',s,s')\) of each tuple. From the defined quantities, the distribution \(\bm{d}^{(a,n, m)}_{AS^2}\) can be defined as:
\begin{align*}
	\bm{d}^{(a,n, m)}_{AS^2} = \mathbb{E}_{\pi, \bm{\nu}}\Bigg[\frac{1}{m} \sum_{t=0}^{n-1} \mathds{1}\{a_{t}=a\}\; \bm{y}_t\ \Big| \mathcal{E}(a,n,m)\Bigg].
\end{align*}
Having clarified this aspect, we can go back to considering the distribution \(\bm{d}^{(a,n, m)}_{AO^2}\).\\
Since this quantity contains elements that are all observable, the associated estimator can be computed by simply counting the realizations of observed tuples from dataset \(\mathcal{G}\) and then dividing them by the number of samples \(m\), as described in Equation~\eqref{def:singleEpEstimator}:
\begin{align*}
		\widehat{\bm{d}}^{(a,n,m)}_{AO^2} = \frac{1}{m} \sum_{t=0}^{n-1} \mathds{1}\{a_{t}=a\} \; \bm{x}_t.
\end{align*}
In the related work of~\citet{Azizzadenesheli2016Reinforcement} a different estimator is employed but in a similar setting: in particular, in their Theorem 13, they consider estimates derived from samples drawn from a POMDP and the estimates are conditioned to a specific action \(a\), as it is for our case.\\
In particular, we are able to show that, with probability at least \(1 - \delta\), we have:
\begin{align}
	\left\|\bm{d}^{(a,n,m)}_{AO^2} - \widehat{\bm{d}}^{(a,n,m)}_{AO^2}\right\|_2 & = \left\|\frac{1}{m}\, \mathbb{E}_{\pi, \bm{\nu}}\left[\frac{1}{m} \sum_{t=0}^{n-1} \mathds{1}\{a_{t}=a\}\; \bm{x}_t\ \Big| \mathcal{E}(a,n,m)\right] - \frac{1}{m} \sum_{t=0}^{n-1} \mathds{1}\{a_{t}=a\} \; \bm{x}_t \right\|_2 \notag \\
		& \le \sqrt{\left( \frac{G(\pi)}{1 - \eta(\pi)}\right)^2 \, \frac{8 \log \big((AO^2+1)/\delta \big)}{m}} \label{eq:concentrationBoundAO_p1}\\
	& \le \frac{G(\pi)}{1 - \eta(\pi)}\sqrt{\frac{8 \log \big(2AO^2/\delta \big)}{m}},\label{eq:concentrationBoundAO_p2}
\end{align}
where the result in~\ref{eq:concentrationBoundAO_p1} combines both a concentration result on matrix estimates appearing in~\citet{tropp2010user} and an analysis on the variance of the samples coming from the Markov chain appearing in~\citet{Azizzadenesheli2016Reinforcement} which shows that the Markovian dependency between samples leads to a further term \(\frac{G(\pi)^2}{(1 - \eta(\pi))^2}\) in the expression of the variance. Here, \(1 \le G(\pi) < \infty\) is the geometric ergodicity while \(0 \le \eta(\pi) < 1\) is the contraction coefficient, also known as Dobrushin coefficient~\citep{krish2016partially}.
Finally, the last inequality in~\ref{eq:concentrationBoundAO_p2} follows from simple algebraic manipulations.\\
\begin{remark}
As observed in the main paper, we point out that an assumption weaker than Assumption~\ref{ass:minElem} can be used in this part of the proof. Indeed, as observed above, the bound in line~\ref{eq:concentrationBoundAO_p1} holds under the geometric ergodicity assumption. For the quantities we estimate, this corresponds in having a policy \(\pi\) that induces a state distribution such that \(d^{(a,n,m)}_S(s) > 0\) for each \(s \in \mathcal{S}\), where we define the state distribution as:
\begin{align*}
    d^{(a,n,m)}_S(s) = \sum_{s' \in \mathcal{S}} d^{(a,n,m)}_{S^2}(s, s'),
\end{align*}
with \(d^{(a,n,m)}_{S^2}(s, s')\) being defined in Equation~\eqref{eq:fromASStoSS}. By assuming that this \emph{action-ergodicity} condition holds for every policy \(\pi_i\) instead of directly using Assumption~\ref{ass:minElem}, the guarantees of this lemma can be preserved.
\end{remark}

Having defined a bound holding for samples coming from a unique distribution, we are ready to extend this result to samples coming from multiple policies. Let us assume that \(k\) different policies are employed and let us denote with \(n_i\) the cardinality of the generated dataset \(\mathcal{G}_i\) and with \(n_i(a)\) the number of tuples from \(\mathcal{G}_i\) starting with action \(a\). By recalling the definitions of the expected distribution and the related estimator reported respectively in Equations~\eqref{eq:combinedDistribution} and~\eqref{eq:combinedEstimator}, we can prove the following relation holding with probability at least \(1 - \delta\):
\begin{align}
	\left\|\bm{d}^{(a,k)}_{AO^2} - \bm{d}^{(a,k)}_{AO^2}\right\|_2 & = \left\|\frac{1}{N_k(a)}\sum_{i=0}^{k-1} n_i(a) \; \bm{d}_{AO^2}^{(a,n_i,n_i(a))} - \frac{1}{N_k(a)}\sum_{i=0}^{k-1} n_i(a) \;\widehat{\bm{d}}_{AO^2}^{(a,n_i,n_i(a))}\right\|_2 \notag \\
	& = \left\|\frac{1}{N_k(a)}\sum_{i=0}^{k-1} n_i(a) \left( \bm{d}_{AO^2}^{(a,n_i,n_i(a))} - \widehat{\bm{d}}_{AO^2}^{(a,n_i,n_i(a))}\right) \right\|_2 \notag \\
	& \le \frac{1}{N_k(a)}\sum_{i=0}^{k-1} n_i(a) \left\| \bm{d}_{AO^2}^{(a,n_i,n_i(a))} - \widehat{\bm{d}}_{AO^2}^{(a,n_i,n_i(a))} \right\|_2 \label{clp:triangle}\\
	& \le \frac{1}{N_k(a)}\sum_{i=0}^{k-1} n_i(a) \; \frac{G(\pi_i)}{1 - \eta(\pi_i)}\sqrt{\frac{8\log \big(2AO^2k/\delta \big)}{n_i(a)}}, \label{clp:unionBound}
\end{align}
where in line~\ref{clp:triangle} we applied the triangle inequality, while in line~\ref{clp:unionBound} we bound the estimation error made on each distribution \(\bm{d}_{AO^2}^{(a,n_i,n_i(a))}\) with probability \(1 - \delta/k\) and apply the union bound.\\
Let us define the new values \(\widetilde{G}\coloneq \underset{i}{\max}\; G(\pi_i)\) and \(\widetilde{\eta}\coloneq \underset{i}{\min}\; \eta(\pi_i)\). This allows us to proceed as follows:
\begin{align}
	\left\|\bm{d}^{(a,k)}_{AO^2} - \bm{d}^{(a,k)}_{AO^2}\right\|_2 & \le \frac{1}{N_k(a)}\sum_{i=0}^{k-1} n_i(a) \; \frac{G(\pi_i)}{1 - \eta(\pi_i)}\sqrt{\frac{8 \log \big(2AO^2k/\delta \big)}{n_i(a)}}\notag \\
	& \le \frac{\widetilde{G}}{N_k(a) \left(1 - \widetilde{\eta}\right)}\sqrt{8\log \big(2AO^2k/\delta \big)} \; \sum_{i=0}^{k-1} n_i(a)\; \sqrt{\frac{1}{n_i(a)}}\label{clp:takeOut}\\
	& = \frac{\widetilde{G}}{N_k(a) \left(1 - \widetilde{\eta}\right)}\sqrt{8 \log \big(2AO^2k/\delta \big)} \; \sum_{i=0}^{k-1} \sqrt{n_i(a)} \notag \\
	& \le \frac{\widetilde{G}}{\left(1 - \widetilde{\eta}\right)}\sqrt{\frac{8 k \log \big(2AO^2k/\delta \big)}{N_k(a)}}\label{clp:cauchy}
\end{align}
where in line~\ref{clp:takeOut} we bound the singular terms in the summation using the definition of \(\widetilde{G}\) and \(\widetilde{\eta}\) and bring them out of the sum, line~\ref{clp:cauchy} is instead obtained by using the definition \(N_k(a)=\sum_{i=0}^{k-1}n_i(a)\) and the Cauchy-Schwartz inequality for which it holds that \(\sum_{i=0}^{k-1} \sqrt{n_i(a)} \le \sqrt{k \; \sum_{i=0}^{k-1} n_i(a)}\).\\
The expression obtained shows that the error of the combined estimator pays a further term \(\sqrt{k}\) in the bound, but it scales with \(\mathcal{O}(1 / \sqrt{N_k(a)})\), analogously as per the single-policy estimator \(\widehat{\bm{d}}_{AO^2}^{(a,n,n_i(a))}\).

\subsection*{Second Part of the Proof}
Let us now focus on the second part. This part shares some similarities with the proof of Lemma 5.2 appearing in~\citet{russo2024efficient}. However, the results are applied to different quantities since (i) the distribution employed here is conditioned on a specific action \(a\), (ii) this distribution is defined with respect to a finite amount of samples and (iii) this distribution combines samples coming from different policies.

Having defined a concentration result on the combined estimator \(\widehat{\bm{d}}^{(a,k)}\), we proceed as follows:

\begin{align}
		\left\| \bm{d}^{(a,k)}_{AS^2} - \widehat{\bm{d}}^{(a,k)}_{AS^2}\right\|_2 = & \left\| \mathbb{B}_a^\dagger \; \left(\bm{d}^{(a,k)}_{AO^2} - \widehat{\bm{d}}^{(a,k)}_{AO^2}\right)\right\|_2 \notag \\
		\le & \left\|\mathbb{B}_a^\dagger \right\|_2 \left\| \bm{d}^{(a,k)}_{AO^2} - \widehat{\bm{d}}^{(a,k)}_{AO^2} \right\|_2 \notag \\
		= & \frac{1}{\sigma_{\min}(\mathbb{B}_a)} \left\| \bm{d}^{(a,k)}_{AO^2} - \widehat{\bm{d}}^{(a,k)}_{AO^2}\right\|_2\\
		\le & \frac{1}{\alpha^2} \left\| \bm{d}^{(a,k)}_{AO^2} - \widehat{\bm{d}}^{(a,k)}_{AO^2}\right\|_2, \label{lp:2normDif}
	\end{align}
	where the first equality can be directly derived from Equation~\eqref{eq:invertedToConditionalASS}, while the first inequality follows by the consistency property of matrices. The last inequality derives from the definition of the block diagonal matrix \(\mathbb{B}_a\), which is composed of submatrices \(\{\mathbb{O}_{a,a'}\}_{a' \in \mathcal{A}}\) for which it holds that \(\sigma_{\min}(\mathbb{O}_{a,a'}) \ge \alpha^2\) for all \((a,a') \in \mathcal{A}^2\). For the properties of block diagonal matrices, it also follows that \(\sigma_{\min}(\mathbb{B}_a) \ge \alpha^2\). Combining this last result with the one in~\eqref{clp:cauchy}, we get with probability at least \(1 - \delta\):
\begin{align}\label{eq:boundDas}
	\left\|\bm{d}^{(a,k)}_{AS^2} - \widehat{\bm{d}}^{(a,k)}_{AS^2} \right\|_2 \le \frac{\widetilde{G}}{\alpha^2 \left(1 - \widetilde{\eta}\right)} \sqrt{\frac{8 k \log \left( 2AO^2k/\delta \right) }{N_k(a)}}
\end{align}
Following the steps of the \emph{Action-wise} OAS estimation algorithm (Algorithm~\ref{alg:AOASEstimation}), the estimated vector \(\widehat{\bm{d}}^{(a,k)}_{AS^2}\in \mathbb{R}^{AS^2}\) is aggregated into the new vector \(\widehat{\bm{d}}^{(a,k)}_{S^2} \in \mathbb{R}^{S^2}\) such that:
\begin{align}
    \widehat{\bm{d}}^{(a,k)}_{S^2}(s,s') = \sum_{a' \in \mathcal{A}} \widehat{\bm{d}}^{(a,k)}_{AS^2}(a',s,s') \qquad \quad \forall s, s' \in \mathcal{S}.
\end{align}

Making use of the Aggregation Lemma appearing in Lemma~\ref{lemma:aggregation}, it is possible to show that the following holds:
\begin{align}\label{lp:aggregation}
\left\|\bm{d}^{(a,k)}_{S^2} - \widehat{\bm{d}}^{(a,k)}_{S^2} \right\|_2 \le \sqrt{A}\; \left\|\bm{d}^{(a,k)}_{AS^2} - \widehat{\bm{d}}^{(a,k)}_{AS^2} \right\|_2.
\end{align}
The next step in the algorithm requires to set to 0 all the negative elements appearing in vector \(\widehat{\bm{d}}^{(a,k)}_{S^2}\). By Assumption~\ref{ass:minElem}, it can easily be observed that all the elements appearing in the real quantity \(\bm{d}^{(a,k)}_{S^2}\) are positive. For this reason, the elements of the non-negative vector \(\widebar{\bm{d}}^{(a,k)}_{S^2}\) obtained by transforming \(\widehat{\bm{d}}^{(a,k)}_{S^2}\) are closer to the real quantities contained in \(\bm{d}^{(a,k)}_{S^2}\). We can thus see that:
\begin{align}\label{lp:transformationBound}
	\left\| \bm{d}^{(a,k)}_{S^2} - \widebar{\bm{d}}^{(a,k)}_{S^2} \right\|_2 \le 	\left\| \bm{d}^{(a,k)}_{S^2} - \widehat{\bm{d}}^{(a,k)}_{S^2} \right\|_2.
\end{align}

For what follows, we will use notation \(\bm{d}^{(a,k)}_{S^2}(s,\cdot) \in \mathbb{R}^{S}\) to denote the subvector of dimension \(S\) containing the quantities \(\bm{d}^{(a,k)}_{S^2}(s,s')\) for each \(s' \in \mathcal{S}\). 

We can then write:
\begin{align}
    \|\mathbb{T}_a - \widehat{\mathbb{T}}_a\|_{F} & = \sqrt{\sum_{s \in \mathcal{S}} \sum_{s' \in \mathcal{S}} \left( \mathbb{T}_a(s,s') - \widehat{\mathbb{T}}_a(s,s') \right)^2} = \sqrt{\sum_{s \in \mathcal{S}} \left\| \mathbb{T}_a(s,\cdot) - \widehat{\mathbb{T}}_a(s,\cdot) \right\|_2^2}\notag\\
    & = \sqrt{\sum_{s \in \mathcal{S}} \left\|\frac{\bm{d}^{(a,k)}_{S^2}(s,\cdot)}{\left\|\bm{d}^{(a,k)}_{S^2}(s,\cdot)\right\|_1} - \frac{\widebar{\bm{d}}^{(a,k)}_{S^2}(s, \cdot)}{\left\|\widebar{\bm{d}}^{(a,k)}_{S^2}(s,\cdot)\right\|_1} \right\|_2^2} \label{lp:sanitycheck}\\
    & \le \sqrt{\sum_{s \in \mathcal{S}} \left\|\frac{\bm{d}^{(a,k)}_{S^2}(s,\cdot)}{\left\|\bm{d}^{(a,k)}_{S^2}(s,\cdot)\right\|_2} - \frac{\widebar{\bm{d}}^{(a,k)}_{S^2}(s,\cdot)}{\left\|\widebar{\bm{d}}^{(a,k)}_{S^2}(s,\cdot)\right\|_2} \right\|_2^2}\label{lp:normRelations}\\
    & \le \sqrt{\sum_{s \in \mathcal{S}} 
        \frac{4 \left\| \bm{d}^{(a,k)}_{S^2}(s,\cdot) - \widebar{\bm{d}}^{(a,k)}_{S^2}(s,\cdot) \right\|_2^2}{\max \Big\{\left\|\bm{d}^{(a,k)}_{S^2}(s,\cdot)\right\|_2, \left\|\widebar{\bm{d}}^{(a,k)}_{S^2}(s,\cdot)\right\|_2\Big\}^2}} \label{lp:lemmaG}\\
    & \le \sqrt{\sum_{s \in \mathcal{S}} 
        \frac{4 \left\| \bm{d}^{(a,k)}_{S^2}(s,\cdot) - \widebar{\bm{d}}^{(a,k)}_{S^2}(s,\cdot) \right\|_2^2}{\left\|\bm{d}^{(a,k)}_{S^2}(s,\cdot)\right\|_2^2}} \notag\\
    & \le \sqrt{\sum_{s \in \mathcal{S}} 
        \frac{4 S \left\| \bm{d}^{(a,k)}_{S^2}(s,\cdot) - \widebar{\bm{d}}^{(a,k)}_{S^2}(s,\cdot) \right\|_2^2}{\left(d_{\min}^{(a)}\right)^2}}\label{lp:norm2toNorm1}\\
    & = \sqrt{\frac{4 S \left\|\bm{d}^{(a,k)}_{S^2} -\widebar{\bm{d}}^{(a,k)}_{S^2}\right\|_2^2}{\left(d_{\min}^{(a)}\right)^2}}\label{lp:eqSubVectoVec}\\ 
    & = \frac{2 \sqrt{S} \left\|\bm{d}^{(a,k)}_{S^2} -\widebar{\bm{d}}^{(a,k)}_{S^2}\right\|_2}{d_{\min}^{(a)}} \notag\\
    & \le \frac{2 \sqrt{S} \left\|\bm{d}^{(a,k)}_{S^2} -\widehat{\bm{d}}^{(a,k)}_{S^2}\right\|_2}{d_{\min}^{(a)}}\label{lp:transformedD}.
\end{align}
Equality in line~\ref{lp:sanitycheck} holds for the definition of the estimated matrix \(\widehat{\mathbb{T}}_a\)\footnote{We assume here that the estimated vectors are such that \(\|\widebar{d}^{(a)}_{S^2}(s,\cdot)\|_1 \neq 0\). However, if this is not the case, instead of nullifying the negative terms of vector \(\widehat{\bm{d}}^{(a)}_{S^2}\), we could simply make the terms positive by a small amount and the result in Equation~\eqref{lp:transformationBound} would still hold.}.
The first inequality in line~\ref{lp:normRelations} derives from the relation between norms \(\left\|\widebar{\bm{d}}^{(a,k)}_{S^2}(s,\cdot)\right\|_1 \ge \left\|\widebar{\bm{d}}^{(a,k)}_{S^2}(s,\cdot)\right\|_2\), while line~\ref{lp:lemmaG} follows from Lemma~\ref{lemma:trulyBatch}.\\ 
Line~\ref{lp:norm2toNorm1} follows from Lemma~\ref{lemma:minActionDist} with \(d_{\min}^{(a)}\) representing the minimum state probability conditioned on action \(a\), which is bounded away from 0 thanks to Assumption~\ref{ass:minElem}. The equality in line~\ref{lp:eqSubVectoVec} is simply obtained by observing that:
	\begin{align*}
		\sum_{s \in \mathcal{S}} \|\bm{d}^{(a,k)}_{S^2}(s,\cdot) - \widebar{\bm{d}}^{(a,k)}_{S^2}(s,\cdot)\|_2^2 = \|\bm{d}^{(a,k)}_{S^2} -\widebar{\bm{d}}^{(a,k)}_{S^2}\|_2^2,
	\end{align*}
	holding by the definition of \(\bm{d}^{(a,k)}_{S^2}\) and \(\widehat{\bm{d}}^{(a,k)}_{S^2}\) respectively, while the last inequality simply uses the bound in~\eqref{lp:transformationBound}.
	
By combining the results obtained in~\eqref{eq:boundDas}, Equation~\eqref{lp:aggregation}, and~\eqref{lp:transformedD}, we obtain the final result holding with probability \(1 - \delta\):
\begin{equation*}
	\|\mathbb{T}_a - \widehat{\mathbb{T}}_a\|_{F} \le \frac{4 \, \widetilde{G}}{\alpha^2 d_{\min}^{(a)} (1 - \widetilde{\eta})}\sqrt{\frac{2k\;SA \,\log \left( 2AO^2k/\delta \right)}{N_k(a)}},
\end{equation*}
which completes the proof. 
\end{proof}

\vspace{0.5cm}

\section{PROOF OF THEOREM~\ref{theorem:algorithmBound}}\label{appendix:theoremProof}
In this section, we will provide the proof for Theorem~\ref{theorem:algorithmBound}. This theorem makes use of the result in Lemma~\ref{lemma:combinedActionDistr} which shows convergence results for samples collected under different policies. The main steps of the proof share similarities with the proof of Theorem 6.1 of~\citet{russo2024efficient}. In particular, we improve over that result by (i) adopting the new \emph{Action-wise} OAS estimator which is able to reuse samples from different episodes, and by (ii) providing a tighter concentration result on the belief error (Lemma~\ref{lemma:BoundOnBeliefErrors}).

\subsection*{Notation and Useful Quantities} 
Before proceeding with the proof, we will need to define some notation that will be useful for what will follow. We define the expected reward of an action \(a_t\) assuming to be in state \(s_t\) as:
\begin{align*}
	\mu(s_t, a_t) = \sum_{o \in \mathcal{O}}r(o) \mathbb{O}_{a_t}(o|s_t) = \bm{r}^\top \mathbb{O}_{a_t}(\cdot|s_t).
\end{align*}
Therefore, we can define the expected reward given a belief state \(b_t\) at time \(t\) when taking action \(a_t\) as:
	\begin{align}
		g(b_t, a_t) = \sum_{s \in \mathcal{S}}\mu(s, a_t) b_t(s) = \bm{\mu}(a_t)b_t = \bm{r}^\top \mathbb{O}_{a_t}b_t,\label{tp:gDefinition}
	\end{align}
	where the last equalities define the expression in matrix notation, with \(\bm{\mu}(a_t)\) being a vector of dimension \(S\) containing the quantity \(\mu(s,a_t)\:\: \forall s \in \mathcal{S}\).
	
	We will use \(\mathbb{T}=\{\mathbb{T}_a\}_{a \in \mathcal{A}}\) to denote the real transition model and \(\mathcal{Q}\) to denote the real POMDP instance.\\
	We will employ \(\widehat{\mathbb{T}}_k=\{\widehat{\mathbb{T}}_{a,k}\}\) to denote the transition model estimated by the \emph{Action-wise} OAS estimation procedure at the beginning of episode \(k\), while we will use \(\mathbb{T}_k=\{\mathbb{T}_{a,k}\}\) to denote the optimistic transition model returned as output by the oracle and actually used during episode \(k\). Analogously, we will denote the estimated and the optimistic POMDP instances at episode \(k\) with \(\widehat{\mathcal{Q}}_k\) and \(\mathcal{Q}_k\) respectively.
	
	We will denote with \(t_k\) the starting time of episode \(k\) and each episode \(k\) will be thus characterized by the timestamps \([t_k,\, t_k+1, \dots, \, t_{k+1}-1]\) with \(t_{k+1}-1\) defining the last timestamp of episode \(k\). For convenience, during the analysis, we will use variable \(E_k\) to characterize the interval associated with the timestamps of episode \(k\) from which we exclude the last timestamp of the episode, namely \(E_k\coloneq [t_k,\, t_k+1, \dots, \, t_{k+1}-2]\). The last sample of each episode \(k\) is excluded from \(E_k\) since it is not entirely used for estimation by the \emph{Action-wise} OAS procedure.
	
	We will also define a probability distribution defined on the belief space as:
\begin{align*}
	U(b_{t+1}|b_t, a) = P_{\mathcal{Q}}(b_{t+1}|b_t, a)
\end{align*} 
	where the probability is defined with respect to the transition model \(\mathbb{T}\) and the observation model \(\mathbb{O}\) referred to the POMDP \(\mathcal{Q}\).
	We will use \(U_k\) to denote a probability distribution defined with respect to the optimistic POMDP \(\mathcal{Q}_k\).\\

Having defined the used notation, we start by reporting the main result of the Theorem here.

\algorithmBound*

\begin{proof}
We recall here the definition of regret as reported in~\eqref{def:regret}:
	\begin{align}\label{eq:regret01}
		\mathcal{R}_T := T\rho^* - \sum_{t=0}^{T-1} r(o_t) = \sum_{t=0}^{T-1}(\rho^* - \mathbb{E}^{\pi}[r(O_t)|\mathcal{F}_{t-1}]) + \sum_{t=0}^{T-1}(\mathbb{E}^{\pi}[r(O_t)|\mathcal{F}_{t-1}] - r(o_t)),
	\end{align}
	where we consider an expectation \(\mathbb{E}^\pi\) taken w.r.t. the true transition model \(\mathbb{T}=\{\mathbb{T}_a\}_{a \in \mathcal{A}}\) and the true observation model \(\mathbb{O}=\{\mathbb{O}_a\}_{a \in \mathcal{A}}\). We use \(\mathcal{F}_{t-1}\) to denote the filtration defined with respect to the events occurring up to time \(t-1\). The second term in the summation defines a martingale. Indeed, by denoting  a stochastic process as:
	\begin{align*}
		X_0=0, \; X_t = \sum_{l=0}^{t-1} (\mathbb{E}^{\pi}[r(O_l)|\mathcal{F}_{l-1}] - r(o_l)),
	\end{align*}
	we can easily see that \(X_t\) represents a martingale. Thus, by applying the Azuma-Hoeffding inequality~\citep{azuma1967weighted} we have that with probability at least \(1 - \delta/4\) we have:
	\begin{align}\label{tp:martingale00}
		\sum_{t=0}^{T-1} (\mathbb{E}^{\pi}[r(O_t)|\mathcal{F}_{t-1}] - r(o_t)) \le \sqrt{2T\ln(4/\delta)}.
	\end{align}
	
	Since action \(A_t\) is adapted to the filtration \(\mathcal{F}_{t-1}\), we have:
	\begin{align*}
		\mathbb{E}^\pi[\mu(S_t, A_t)|\mathcal{F}_{t-1}] = g(b_t,A_t),
	\end{align*}
	where function \(g(\cdot,\cdot)\) is defined in Equation~\ref{tp:gDefinition}, while the belief \(b_t\) is computed using the true transition and observation matrices and actions are taken according to policy \(\pi\). Using analogous notation, we will denote the expected instantaneous reward assuming to have computed the belief using the estimated transition probability \(\mathbb{T}_{a,k}\) as:
	\begin{align*}
		\mathbb{E}^\pi_k[\mu(S_t, A_t)|\mathcal{F}_{t-1}] = g(b_t^k,A_t).
	\end{align*}
	Given the defined quantities, we can rewrite the first term of~\eqref{eq:regret01} as:
	\begin{align}
		\sum_{t=0}^{T-1}(\rho^* - \mathbb{E}^{\pi}[r(O_t)|\mathcal{F}_{t-1}]) = \sum_{t=0}^{T-1}(\rho^* - \mathbb{E}^{\pi}[\mu(S_t, A_t)|\mathcal{F}_{t-1}]) = \sum_{t=0}^{T-1}(\rho^* - g(b_t,A_t)).
	\end{align}
	By following the procedure described in the \emph{Action-wise} OAS-UCRL algorithm, at the beginning of each episode \(k\), an optimistic POMDP \(\mathcal{Q}_{k}\) is chosen from the set of possible POMDPs determined by the confidence region \(\mathcal{C}_k(\delta_k)\). We recall that the optimistic POMDP \(\mathcal{Q}_{k}\) is defined by the optimistic transition model \(\mathbb{T}_k=\{\mathbb{T}_{a,k}\}_{a \in \mathcal{A}}\) provided by the oracle, and the real observation model.\\
	The confidence region \(\mathcal{C}_k(\delta_k)\) of the transition model in episode \(k\) can be associated with the confidence regions \(\mathcal{C}_{a,k}(\delta_{a,k})\) of each action transition model. Each confidence region \(\mathcal{C}_{a,k}(\delta_{a,k})\) is centered in the estimated action transition matrix \(\widehat{\mathbb{T}}_{a,k}\) and is such that \(P(\mathbb{T}_a \in \mathcal{C}_{a,k}(\delta_{a,k}))\ge 1 - \delta_{a,k}\).
	
	Now we consider two possible events: the \emph{good event} which considers the case where for all episodes \(k\), the true POMDP is contained in the confidence sets \(\mathcal{C}_k(\delta_k)\) and the \emph{failure event} which denotes the complementary event. The \emph{good event} implies that all the real action transition models \(\mathbb{T}_a\) are contained in their confidence region \(\mathcal{C}_{a,k}(\delta_{a,k})\) for all episodes \(k\).\\
	We set the confidence level of the transition model in episode \(k\) as \(\delta_{k}\coloneq \frac{\delta}{k^3}\) and set the confidence level of each action transition model in episode \(k\) as \(\delta_{a,k}\coloneq \frac{\delta}{Ak^3}\).\\
We can thus bound the probability of the \emph{failure event} as:
\begin{align*}
	P(\mathcal{Q} \notin \mathcal{C}_k(\delta_k), \text{for some k}) & = P(\mathbb{T}_a \notin \mathcal{C}_{a,k}(\delta_{a,k}), \text{for some        } a,\,k)\\
	& \le \sum_{k=1}^{K-1} \sum_{a \in \mathcal{A}} \delta_{a,k} = \sum_{k=1}^K \underbrace{A \frac{\delta}{A k^3}}_{\delta_k} \le \frac{3}{2}\delta,
\end{align*}
From this formulation, it appears that the \emph{good event} holds with probability at least \(1 - \frac{3}{2}\delta\). When this is the case, we have that \(\rho^* \le \rho^k\) for any \(k\) since the optimal average reward is taken from the optimistic POMDP \(\mathcal{Q}_{k}\).\\
	
We can now bound the regret under the \emph{good event} during the different \(K\) episodes as:
\begin{align}
	\sum_{t=0}^{T-1}(\rho^* - g(b_t,A_t)) & \le K + \sum_{k=0}^{K-1} \sum_{t \in E_k} (\rho^* - g(b_t,A_t)) \notag \\
	 & \le K + (T_0 - 1) + \sum_{k=1}^{K-1} \sum_{t \in E_k} \left(\rho^k - g(b_t,A_t)\right) \notag \\
	& = K + \sum_{a \in \mathcal{A}}n_0(a) + \sum_{k=1}^{K-1} \sum_{t \in E_k} \underbrace{\Big[\rho^k - g(b_t^k,A_t)\Big]}_{\text{First Term}} + \underbrace{\Big[g(b_t^k,A_t) - g(b_t,A_t)\Big]}_{\text{Second Term}}\label{tp:001},
\end{align}
where we have rewritten the summation by highlighting the different episodes \(K\). In particular, for each episode \(k\), we use interval \(E_k\) which excludes the last timestamp of that episode, and the term \(K\) appearing in the first inequality is obtained by assuming to pay maximum regret for each excluded sample.\\
In the second inequality instead, we make explicit the length of the first episode \(T_0\) for which we assume to pay maximum regret and from which we subtract \(1\) (which is the last sample of the episode already counted in the \(K\) term). In the last equality, we rewrite the length of the first episode as the sum of counts of the chosen actions.\\
For the moment, we will not consider the terms \(K\) and \(\sum_{a \in \mathcal{A}}n_0(a)\) but we will focus on the different terms appearing in the summation. 
\vspace{0.15cm}
\subsection*{Analysis of the First Term in~\ref{tp:001}}
As a first step, we will consider the first term appearing in the summation in~\ref{tp:001}. It can be bounded by using the Bellman equation reported in Equation~\eqref{eq:Bellman} for the optimistic belief MDP. By using the probability distribution \(U\) defined on the next belief (see Notation section), we can rewrite the Bellman equation as follows:
	\begin{align*}
		\rho^k + v_k(b_t^k) & = g(b_t^k,A_t) + \int_{b_{t+1} \in \mathcal{B}} v_k(b_{t+1})U_k(\,db_{t+1}|b_t^k,A_t)\\
		& = g(b_t^k,A_t) + \langle U_k(\cdot|b_t^k, A_t), v_k(\cdot) \rangle.
	\end{align*}
	Given that the value function \(v_k\) satisfies the Bellman Equation, a shifted version \(v_k + c\bm{1}\) of the bias function would satisfy it as well. From this consideration, we can assume that 
	\(\|v_k\|_\infty \le \text{span}(v_k)/2\). By using the result in Proposition~\ref{prop:uniformBoundBias} reported in~\citet{zhou2021regime}, we are able to bound the span of \(v_k\), where the span is defined as \(span(v_k):=\max_{b \in \mathcal{B}}v_k(b) - \min_{b \in \mathcal{B}}v_k(b)\). In particular, we use the finite constant \(D\) to bound the span. From these considerations, we can write:
	\begin{align}\label{tp:biasModuleBound}
		\|v_k\|_\infty \le \frac{\text{span}(v_k)}{2} \le \frac{D}{2}.
	\end{align}
	By combining the elements reported so far, for the first term in the summation of~\ref{tp:001}, we can write:
	\begin{align}
		\sum_{k=1}^{K-1} \sum_{t \in E_k} (\rho^k - g(b_t^k,A_t)) = & \sum_{k=1}^{K-1} \sum_{t \in E_k} \left(-v_k(b_t^k) + \langle U_k(\cdot|b_t^k, A_t), v_k(\cdot) \rangle \right) \notag\\
		= & \sum_{k=1}^{K-1} \sum_{t \in E_k} \left(-v_k(b_t^k) + \langle U(\cdot|b_t^k, A_t), v_k(\cdot) \rangle \right) + \left(\langle U_k(\cdot|b_t^k, A_t) - U(\cdot|b_t^k, A_t), v_k(\cdot) \rangle\right)\label{tp:boundVQ},
	\end{align}
	where the first equality is obtained from the Bellman Equation, while the last equality is obtained by adding and subtracting the term \(\langle U(\cdot|b_t^k, A_t), v_k(\cdot) \rangle\) for each time step t and we recall that \(U(\cdot|b_t^k, A_t)\) represents the probability distribution over the belief at the next step \(t+1\) under the true POMDP instance \(\mathcal{Q}\), while \(U_k(\cdot|b_t^k, A_t)\) represents this probability distribution under the optimistic instance \(\mathcal{Q}_k\).
	
	For the first term of~\ref{tp:boundVQ}, we have:
	\begin{align}
		\sum_{k=1}^{K-1} \sum_{t \in E_k} \left(-v_k(b_t^k) + \langle U(\cdot|b_t^k, A_t), v_k(\cdot) \rangle \right)
		= \; & \sum_{k=1}^{K-1} \sum_{t \in E_k} \left(-v_k(b_t^k) + v_k(b_{t+1}^k)\right) + \left(-v_k(b_{t+1}^k) + \langle U(\cdot|b_t^k, A_t), v_k(\cdot) \rangle \right) \notag\\
		= \; & \sum_{k=1}^{K-1} \Big( -v_k(b_{s_k}^k) + v_k(b_{e_k+1}^k)\Big) + \sum_{k=1}^{K-1} \sum_{t \in E_k} \mathbb{E}^\pi [v_k(b_{t+1}^k|\mathcal{F}_t)] - v_k(b_{t+1}^k),\notag
	\end{align}
	where the first term appearing in the last equality represents a telescopic summation. For each episode \(k\), the terms appearing in this summation are respectively the value of the bias function of the belief \(b^k_{s_k}\) (with \(s_k\) denoting the starting time step of episode \(k\)) and the value of the bias function of the belief \(b^k_{e_k+1}\) (with \(e_k\) denoting the last time step of \(E_k\)).\\
The second term appearing in the last equality is instead obtained by showing that:
	\begin{align*}
		\langle U(\cdot|b_t^k, A_t), v_k(\cdot) \rangle = \int_{b_{t+1} \in \mathcal{B}} v_k(b_{t+1})U(\,db_{t+1}|b_t^k,A_t) = \mathbb{E}^\pi[v_k(b_{t+1}^k|b_t^k)] =  \mathbb{E}^\pi [v_k(b_{t+1}^k|\mathcal{F}_t)].
	\end{align*}
	By recalling Proposition~\ref{prop:uniformBoundBias} to bound the span of the bias function, we can easily see that:
	\begin{align}
		\sum_{k=1}^{K-1} -v_k(b_{s_k}^k) + v_k(b_{e_k+1}^k) \le \sum_{k=1}^{K-1} D = (K-1)\;D.\label{tp:boundKD}
	\end{align}
	By applying analogous considerations as those used for bounding~\ref{tp:martingale00}, we can state that this sum of differences defines a martingale. Thus, with probability at least \(1 - \delta/4\), we have that:
	\begin{align}
		\sum_{k=1}^{K-1} \sum_{t \in E_k} \mathbb{E}^\pi [v_k(b_{t+1}^k|\mathcal{F}_t)] - v_k(b_{t+1}^k) \le D \sqrt{2T \ln \left(\frac{4}{\delta}\right)}.\label{tp:martingale03}
	\end{align}
	By combining the previous considerations, we can bound the first term in~\ref{tp:boundVQ} as:
	\begin{align}\label{tp:telescopingAndMartingale}
		\sum_{k=1}^{K-1} \sum_{t \in E_k} \left(-v_k(b_t^k) + \langle U(\cdot|b_t^k, A_t), v_k(\cdot) \rangle \right) \le (K-1)\,D + D \sqrt{2T \ln \left(\frac{4}{\delta}\right)}.
	\end{align}
	We can now proceed in bounding the second term appearing in~\ref{tp:boundVQ}. Before going on with this step, we need to introduce functions \(H(b_t, a_t, o_t)\) and \(H_k(b_t, a_t, o_t)\) which return the belief at the next time step \(b_{t+1}\) given the current belief \(b_t\), the action taken \(a_t\) and the received observation \(o_t\) using the real \(\mathbb{T}_a\) and the optimistic transition matrix \(\mathbb{T}_{a,k}\), respectively.\\
	By analyzing each term appearing in the second summation of~\ref{tp:boundVQ}, we get:
	\begin{align}
		\langle U_k(\cdot|b_t^k, A_t) - U(\cdot|b_t^k, A_t), v_k(\cdot) \rangle
		\le \; & \Bigg\lvert \int_{\mathcal{B}} v_k(b')U_k(\,db'|b_t^k,A_t) - \int_{\mathcal{B}} v_k(b')U(\,db'|b_t^k,A_t)\Bigg\rvert \notag \\
		= \; & \Bigg\lvert \sum_{o_t \in \mathcal{O}} v_k \left( H_k(b_t^k,A_t,o_t)\right) P(o_t|b_t^k, A_t) - \sum_{o_t \in \mathcal{O}} v_k \left(H(b_t^k,A_t,o_t)\right) P(o_t|b_t^k, A_t) \Bigg\rvert \notag\\
		= \; & \Bigg\lvert \sum_{o_t \in \mathcal{O}} \left[ v_k\left(H_k(b_t^k,A_t,o_t)\right) - v_k\left(H(b_t^k,A_t,o_t)\right) \right] P(o_t|b_t^k, A_t)\Bigg\rvert \notag\\
		\le \; & \sum_{o_t \in \mathcal{O}} \Big\lvert v_k(H_k(b_t^k,A_t,o_t)) - v_k(H(b_t^k,A_t,o_t)) \Big\rvert \, P(o_t|b_t^k, A_t) \notag\\
		\le \; & \sum_{o_t \in \mathcal{O}} \frac{D}{2}\left | H_k(b_t^k,A_t,o_t) -H(b_t^k,A_t,o_t)\right| \,P(o_t|b_t^k, A_t)\notag \\
		\le \; & \sum_{o_t \in \mathcal{O}} \frac{D}{2} \Big(L_1\, \|\mathbb{T}_{A_t} - \mathbb{T}_{A_t,k}\|_F \Big) \, P(o_t|b_t^k, A_t) \tag{Corollary~\ref{corollary:oneStepBeliefBound}}\\
		= \; & \frac{D L_1}{2} \|\mathbb{T}_{A_t} - \mathbb{T}_{A_t,k}\|_F,\label{tp:Qdiff04}
	\end{align}
	where in the first equality we have explicitly decoupled the stochasticity induced by the observation from the deterministic update of the belief \(b'\) at the next step through the \(H\) and \(H_k\) functions. The second inequality is simply obtained by using the triangle inequality, while the third inequality is obtained using the bound on the bias span appearing in~\ref{tp:biasModuleBound}.
The last inequality is instead obtained from Corollary~\ref{corollary:oneStepBeliefBound} bounding the one-step error of the belief vector updated using different transition matrices. Here, we introduce constant \(L_1=\frac{4(1 - \epsilon)}{\epsilon^2}\) derived from the corollary. 
	
	By combining the results obtained so far in~\ref{tp:telescopingAndMartingale} and~\ref{tp:Qdiff04}, we are able to bound the first term appearing in the summation of~\ref{tp:001} as: 
	\begin{align}
		\sum_{k=1}^{K-1} \sum_{t \in E_k} (\rho^k - g(b_t^k,A_t)) & \le (K-1)\,D + D \sqrt{2T \ln \left(\frac{4}{\delta}\right)} + \sum_{k=1}^{K-1} \sum_{t \in E_k} \frac{D L_1}{2} \|\mathbb{T}_{A_t} - \mathbb{T}_{A_t,k}\|_F \notag \\
		& = (K-1)\,D + D \sqrt{2T \ln \left(\frac{4}{\delta}\right)} + \frac{D L_1}{2} \sum_{k=1}^{K-1} \sum_{a \in \mathcal{A}} n_k(a) \,\|\mathbb{T}_{a} - \mathbb{T}_{a,k}\|_F\label{tp:BoundFirstTerm}.
	\end{align}
\vspace{0.15cm}
	
\subsection*{Analysis of the Second Term in~\ref{tp:001}}
	We can now focus on the second term appearing in the summation of~\ref{tp:001}. We have that:
	\begin{align}
    \sum_{k=1}^{K-1} \sum_{t \in E_k} (g(b_t^k,A_t) -  g(b_t,A_t)) & \le \sum_{k=1}^{K-1} \sum_{t \in E_k} \|\bm{r}^\top \mathbb{O}_{A_t}\|_\infty \|b_t^k - b_t\|_1 \notag\\
    & \le \sum_{k=1}^{K-1} \sum_{t \in E_k} \|b_t^k - b_t\|_1, \label{lemma:beliefErrorBound}
\end{align}
where we use Holder's inequality in the first passage, while the second inequality considers that \(\|\bm{r}^\top \mathbb{O}_{A_t}\|_\infty \le 1\).

The expression we use here to bound line~\ref{lemma:beliefErrorBound} uses a new result which we present in Lemma~\ref{lemma:BoundOnBeliefErrors} that improves over the result employed in~\citet{russo2024efficient} (see their Proposition H.3).\\
In particular, Lemma~\ref{lemma:BoundOnBeliefErrors} show that:
\begin{align}
    \sum_{k=1}^{K-1} \sum_{t \in E_k} \|b_t^k - b_t\|_1 & \le \sum_{k=1}^{K-1} \bigg[L + L \sum_{a \in \mathcal{A}} n_k(a) \|\mathbb{T}_{a} - \mathbb{T}_{a,k}\|_F \bigg] \notag\\
    & = (K-1)L + L \sum_{k=1}^{K-1}\sum_{a \in \mathcal{A}} n_k(a) \|\mathbb{T}_{a} - \mathbb{T}_{a,k}\|_F\label{tp:boundSecondTerm},
\end{align}
with constant \(L \coloneq \frac{4(1 - \epsilon)^2}{\epsilon^3}\) defined in the lemma.
\vspace{0.15cm}

\subsection*{Merge of Obtained Results to Bound Line~\ref{tp:001}}
By merging the results obtained in~\ref{tp:BoundFirstTerm} and in~\ref{tp:boundSecondTerm}, we bound line~\ref{tp:001} as follows:
\begin{align}
\sum_{t=0}^{T-1}(\rho^* - g(b_t,A_t)) & \le K + \sum_{a \in \mathcal{A}}n_0(a) + (K-1)\,(D+L) + D \sqrt{2T \ln \left(\frac{4}{\delta}\right)} + \notag\\  
& \qquad \; + \sum_{k=1}^{K-1} \sum_{a \in \mathcal{A}} n_k(a) \,\Bigg( \frac{D L_1}{2} \|\mathbb{T}_{a} - \mathbb{T}_{a,k}\|_F + L \|\mathbb{T}_{a} - \mathbb{T}_{a,k}\|_F \Bigg) \notag\\
& \le K + \sum_{a \in \mathcal{A}}n_0(a) + (K-1)\,(D+L) + D \sqrt{2T \ln \left(\frac{4}{\delta}\right)} + \frac{L(2+D)}{2} \sum_{k=1}^{K-1} \sum_{a \in \mathcal{A}} n_k(a) \,\|\mathbb{T}_{a} - \mathbb{T}_{a,k}\|_F \label{tp:boundthreesums}
\end{align}
where in the last inequality we used \(L_1 \le L\).

Let us now consider the last quantity appearing in line~\ref{tp:boundthreesums} and let us disregard for the moment the multiplicative part \(L(2+D)/2\). We proceed with the analysis: 
\begin{align}
    \sum_{k=1}^{K-1} \sum_{a \in \mathcal{A}} n_k(a) \, \|\mathbb{T}_{a} - \mathbb{T}_{a,k}\|_F & \le \sum_{k=1}^{K-1} \sum_{a \in \mathcal{A}} \frac{4 \, \widetilde{G} \,n_k(a)}{\alpha^2 d_{\min}^{(a,k)} (1 - \widetilde{\eta})}\sqrt{\frac{2 kSA \,\log \left( 2AO^2k/\delta_{a,k} \right)}{N_k(a)}} \notag\\
    & = \sum_{k=1}^{K-1} \sum_{a \in \mathcal{A}} \frac{4 \, \widetilde{G} \,n_k(a)}{\alpha^2 d_{\min}^{(a,k)} (1 - \widetilde{\eta})}\sqrt{\frac{2kSA \,\log \left( 2A^2O^2k^4/\delta \right)}{N_k(a)}} \notag\\
    & \le \frac{4 \, \widetilde{G}}{\alpha^2 \widetilde{d}_{\min} (1 - \widetilde{\eta})}\sqrt{2KSA \,\log \left( \frac{2A^2O^2K^4}{\delta} \right)} \sum_{k=1}^{K-1} \sum_{a \in \mathcal{A}} \frac{n_k(a)}{\sqrt{N_k(a)}},\label{bound:transMatErrorOverEpisodes}
\end{align}
where the first inequality holds by using Lemma~\ref{lemma:combinedActionDistr} and recalling that we are under the \emph{good event}. In the first equality, we make explicit the confidence level \(\delta_{a,k}=\frac{\delta}{Ak^3}\) while in the last expression, we use \(k \le K\) and we set \(\widetilde{d}_{\min}\coloneq \underset{k}{\min}\,\underset{a \in \mathcal{A}}{\min}\,d^{(a,k)}_{\min}\), which is always bounded away from 0 because of Assumption~\ref{ass:minElem}.

Let us now focus on the summation appearing on the right side of the bound in~\ref{bound:transMatErrorOverEpisodes}. At this point, we can also include the action counts associated with episode \(0\). 
It follows that:
\begin{align*}
    \sum_{a \in \mathcal{A}} n_0(a) + \sum_{a \in \mathcal{A}}\sum_{k=1}^{K-1} \frac{n_k(a)}{\sqrt{N_k(a)}} & = \sum_{a \in \mathcal{A}}\sum_{k=0}^{K-1} \frac{n_k(a)}{\sqrt{\max\{1, N_k(a)\}}}\\ 
    & \le \sum_{a \in \mathcal{A}} \big( \sqrt{2} + 1 \big) \sqrt{N_{K}(a)} \tag{Lemma~\ref{lemma:boundJaksch}}\\    
    & \le \big( \sqrt{2} + 1 \big) \sqrt{AT}, \tag{Cauchy-Schwarz inequality}    
\end{align*}
where in the first equality we bring the summation on the terms \(n_0(a)\) in the summation over the episodes by including the \(\max\) at the denominator. The successive inequality is due to Lemma~\ref{lemma:boundJaksch} taken from~\citet{jaksch2010near}, while the last expression is simply obtained by the Cauchy-Schwarz inequality and noting that \(\sum_{a \in \mathcal{A}}N_K(a) = T - K \le T\).

We are now able to rewrite the bound in~\ref{tp:boundthreesums} as:
\begin{align}
    \sum_{t=0}^{T-1}(\rho^* - g(b_t,A_t)) & \le K + (K-1)\,(D+L) + D \sqrt{2T \ln \left(\frac{4}{\delta}\right)} + \notag \\
    & \qquad \; + (\sqrt{2}+1) \frac{2 L (2+D)\, \widetilde{G}}{\alpha^2 \widetilde{d}_{\min} (1 - \widetilde{\eta})}\sqrt{2KSA^2T \,\log \left( \frac{2A^2O^2K^4}{\delta} \right)} \notag\\
    & \le 2K(D+L) + D \sqrt{2T \ln \left(\frac{4}{\delta}\right)} + \frac{6 L (2+D)\, \widetilde{G}}{\alpha^2 \widetilde{d}_{\min} (1 - \widetilde{\eta})}\sqrt{2KSA^2T \,\log \left( \frac{2A^2O^2K^4}{\delta} \right)}\label{tp:combinedResults}
\end{align}

\vspace{0.15cm}

\subsection*{Final Regret Result}
By recalling the definition of the regret in line~\ref{eq:regret01}, we are finally able to combine the result appearing in line~\ref{tp:combinedResults} and the result on the martingale in line~\ref{tp:martingale00} using a union bound. Indeed, with probability at least \(1 - 2\delta\), we have:
\begin{align}
	\mathcal{R}_T & \le 2K(D+L) + D \sqrt{2T \ln \left(\frac{4}{\delta}\right)} + \sqrt{2T \ln \left(\frac{4}{\delta}\right)} + \frac{6 L (2+D)\, \widetilde{G}}{\alpha^2 \widetilde{d}_{\min} (1 - \widetilde{\eta})}\sqrt{2KSA^2T \,\log \left( \frac{2A^2O^2K^4}{\delta} \right)}\notag \\
	& \le 2K(D+L) + 2D \sqrt{2T \ln \left(\frac{4}{\delta}\right)} + \frac{6 \,C \,(2+D)\, \widetilde{G}}{\alpha^2 \widetilde{d}_{\min}}\sqrt{2KSA^2T \,\log \left( \frac{2A^2O^2K^4}{\delta} \right)},
\end{align}
where in the last inequality we used that \(D \ge 1\) and defined a new constant \(C\coloneq \frac{4(1-\epsilon)^3}{\epsilon^4}\) by using that
\begin{align*}
	\frac{L}{1 - \widetilde{\eta}} \le 	\frac{L (1 - \epsilon)}{\epsilon} = \frac{4(1-\epsilon)^3}{\epsilon^4} \eqcolon C,
\end{align*}
where the first inequality holds for the properties of the contraction coefficient since \(\widetilde{\eta}\le 1 - \frac{\epsilon}{1 - \epsilon}\) and the following equality follows from the definition of \(L\) in Corollary~\ref{corollary:oneStepBeliefBound}.\\
This bound on the regret shows an intricate dependence on the problem parameters and a \(\widetilde{\mathcal{O}}(\sqrt{T})\) dependence on time \(T\), while the dependency on the number of episodes \(K\) is linear. It can be shown that under the episode termination condition employed in the \emph{Action-wise} OAS Algorithm, the number of episodes can be bounded in the worst case by the following quantity:
\begin{align*}
	K \le A \log (T/A),
\end{align*}
having logarithmic dependence on time \(T\). By using this last result, we are finally able to provide the final expression of the regret, holding with probability at least \(1 - 2\delta\):
\begin{align}
	\mathcal{R}_T & \le 2A \log (T/A)(D+L) + 2D \sqrt{2T \ln \left(\frac{4}{\delta}\right)} + \frac{6 C (2+D)\, \widetilde{G}}{\alpha^2 \widetilde{d}_{\min}}\sqrt{2SA^3T\, \log(T/A) \,\log \left( \frac{2A^6O^2\log^{4}(T/A)}{\delta} \right)}.
\end{align}
From the formulation above, we can simplify the expression obtaining:
\begin{align*}
	\mathcal{R}_T \le \mathcal{O}\left( \frac{CD\widetilde{G}}{\alpha^2 \widetilde{d}_{\min}}\sqrt{SA^3T\,\log T\,\log O}\right).
\end{align*}
This last step completes the proof.
\end{proof}

\vspace{0.5cm}

\section{CONCENTRATION BOUND OF BELIEF VECTOR UNDER DIFFERENT ACTION MATRICES}\label{appendix:beliefNewConcentration}
We present here Lemma~\ref{lemma:BoundOnBeliefErrors} which will be helpful for the theoretical analysis developed in Theorem~\ref{theorem:algorithmBound}.

\begin{lemma}[Bound on Sum of Belief Errors]\label{lemma:BoundOnBeliefErrors}
	Let \(\mathcal{Q}\) be a POMDP instance satisfying Assumption~\ref{ass:minElem}. Let \(\mathbb{T}=\{\mathbb{T}_a\}_{a \in \mathcal{A}}\) be the transition model and let \(\widehat{\mathbb{T}}=\{\widehat{\mathbb{T}}_a\}_{a \in \mathcal{A}}\) be its estimate. Let a sequence of actions \((a_t)_{t=0}^{T}\) be taken while interacting with the environment and let \(b\) and \(\widehat{b}\) denote the real and estimated belief vector updated using respectively the real and the estimated transition model according to Equation~\eqref{eq:beliefUpdate}. It follows that:
\begin{align*}
	\sum_{t=0}^{T} \|\widehat{b}_t - b_t\|_1 \le L + L \sum_{a \in \mathcal{A}} n(a) \, \|\mathbb{T}_{a} - \widehat{\mathbb{T}}_{a}\|_F 
\end{align*}
where we use constant \(L \coloneq \frac{4 (1 - \epsilon)^2}{\epsilon^3}\), with \(\epsilon\) being the minimum action probability appearing in Assumption~\ref{ass:minElem}, while \(n(a)\) represents the number of times each action \(a \in \mathcal{A}\) has been chosen during the interaction with the horizon.
\end{lemma}
\begin{proof}
Let \(\widehat{b}_t\) and \(b_t\) be the estimated and real belief vector at time \(t\) updated using Equation~\ref{eq:beliefUpdate}, each one using respectively the estimated and real transition model. From a belief decomposition reported in~\citet{deCastroConsistent2017}, it is possible to express the belief error as a sum of the errors of the chosen action transition matrices, as follows:
\begin{align}
    \|\widehat{b}_t - b_t\|_1 & \le \frac{4 \eta^{t}\|\widehat{b}_0 - b_0\|_2}{\epsilon} + \frac{4(1 - \epsilon)}{\epsilon^2} \sum_{l=0}^{t-1} \eta^{t-l-1} \|\mathbb{T}_{a_{l}} - \widehat{\mathbb{T}}_{a_l}\|_F, \label{eq:beliefBoundwithRhoBeliefDep}\\
    & \le \frac{8 \eta^{t}}{\epsilon} + \frac{4(1 - \epsilon)}{\epsilon^2} \sum_{l=0}^{t-1} \eta^{t-l-1} \|\mathbb{T}_{a_{l}} - \widehat{\mathbb{T}}_{a_l}\|_F,\label{eq:beliefBoundwithRhoAnyBelief}
\end{align}
with \(\eta\) being defined as \(\eta = 1 - \frac{\epsilon}{1 - \epsilon}\), while in the second inequality we simply use that \(\|\widehat{b}_0 - b_0\|_2 \le \|\widehat{b}_0 - b_0\|_1 \le 2\).

Basically, this bound states that the error in the belief depends on the sequence of actions taken and a higher contribution is given to the error associated with the most recent actions since the contribution of the error of each action decreases geometrically with time. 

Let us consider now the sequence of actions and observations seen during the interaction and let us denote it with \((a_0, o_0, a_1, o_1, \dots, a_{t}, o_{t})\). First of all, we highlight that the last action-observation pair \((a_{t}, o_{t})\) does not influence the update of \(b_t\) but will influence \(b_{t+1}\) that does not appear in the summation, hence the last tuple \((a_{t}, o_{t})\) will not influence the final result.\\ 
Let us now make explicit the expression in~\ref{eq:beliefBoundwithRhoAnyBelief} for different values of the belief. For readability, we will use the constant term \(C = \frac{2 (1 - \epsilon)}{\epsilon^2}\).
\begin{align*}
    & \|\widehat{b}_0 - b_0\|_1 \le \frac{8}{\epsilon},\\
    & \|\widehat{b}_1 - b_1\|_1 \le \frac{8 \eta}{\epsilon} + C\|\mathbb{T}_{a_0} - \widehat{\mathbb{T}}_{a_0}\|_F,\\
    & \|\widehat{b}_2 - b_2\|_1 \le \frac{8 \eta^2}{\epsilon} + \eta C \|\mathbb{T}_{a_0} - \widehat{\mathbb{T}}_{a_0}\|_F + C \|\mathbb{T}_{a_1} - \widehat{\mathbb{T}}_{a_1}\|_F,\\
    & \|\widehat{b}_3 - b_3\|_1 \le \frac{8 \eta^3}{\epsilon} + \eta^2 C \|\mathbb{T}_{a_0} - \widehat{\mathbb{T}}_{a_0}\|_F + \eta C \|\mathbb{T}_{a_1} - \widehat{\mathbb{T}}_{a_1}\|_F + C \|\mathbb{T}_{a_2} - \widehat{\mathbb{T}}_{a_2}\|_F,\\
    & \vdots\\
    & \|\widehat{b}_{t} - b_{t}\|_1 \le \frac{8 \eta^t}{\epsilon} + \eta^{t-1} C\|\mathbb{T}_{a_0} - \widehat{\mathbb{T}}_{a_0}\|_F +  \eta^{t-2} C\|\mathbb{T}_{a_1} - \widehat{\mathbb{T}}_{a_1}\|_F + \dots + C\|\mathbb{T}_{a_{t-1}} - \widehat{\mathbb{T}}_{a_{t-1}}\|_F.\\
\end{align*}
By reading the expression above along a vertical direction, we can bound the sum of the belief errors across various interaction steps as follows:
\begin{align}
    \sum_{t=0}^T \|\widehat{b}_t - b_t\|_1 & \le \sum_{t=0}^T \frac{4 \eta^t\, \|\widehat{b}_0 - b_0\|_2}{\epsilon} + \frac{4 (1 - \epsilon)}{\epsilon^2} \sum_{t=0}^T \sum_{l=0}^{t-1} \eta^{t-l-1} \|\mathbb{T}_{a_l} - \widehat{\mathbb{T}}_{a_l}\|_F \notag\\
	& \le \frac{8 (1 - \epsilon)}{(1 - \eta) \epsilon} + \frac{4(1 - \epsilon)}{(1 - \eta) \epsilon^2} \sum_{a \in \mathcal{A}} n(a) \,\|\mathbb{T}_{a} - \widehat{\mathbb{T}}_{a}\|_F \notag\\ 
	& \le \frac{4 (1 - \epsilon)}{(1 - \eta) \epsilon^2} + \frac{4(1 - \epsilon)}{(1 - \eta) \epsilon^2} \sum_{a \in \mathcal{A}} n(a) \,\|\mathbb{T}_{a} - \widehat{\mathbb{T}}_{a}\|_F \notag\\ 
	& = \frac{4 (1 - \epsilon)^2}{\epsilon^3} + \frac{4(1 - \epsilon)^2}{\epsilon^3} \sum_{a \in \mathcal{A}} n(a) \,\|\mathbb{T}_{a} - \widehat{\mathbb{T}}_{a}\|_F,  \label{eq:finalErrorBound}
\end{align}
where the first term in the second inequality simply for the bound on geometric series and using that \(\|\widehat{b}_0 - b_0\|_2 \le 2\), while the second term holds since it can be noted that the contribution on the error of each action \(a\) depends on the number of times it is pulled \(n(a)\)\footnote{We highlight here that this count does not consider the last action \(a_T\) since it does not influence the bound.} and the associated error \(\|\mathbb{T}_a - \widehat{\mathbb{T}}_a\|_F\) scaled by at most by \(1 / 1 - \eta\).\\
The third inequality holds for any non-trivial problem instance having a number of states \(S \ge 2\), while the last expression holds by substitution of \(\eta\).\\
The statement of the lemma simply follows by defining constant \(L \coloneq \frac{4 (1 - \epsilon)^2}{\epsilon^3}\).
\end{proof}

In the following, we present a corollary that derives from the considerations reported in the proof of Lemma~\ref{lemma:BoundOnBeliefErrors}.
\begin{corollary}(One-step Belief Bound)\label{corollary:oneStepBeliefBound}
Let \(\mathcal{Q}\) be a POMDP instance satisfying Assumption~\ref{ass:minElem}. Let us denote with \(\mathbb{T}_a\) and \(\widehat{\mathbb{T}}_a\) respectively the real and estimated transition matrix related to action \(a\). Starting from a common belief vector \(b_0\), and choosing action \(a \in \mathcal{A}\), the one-step error in the estimated belief vector can be bounded as:
\begin{equation*}
	\|\widehat{b}_1 - b_1\|_1 \le L_1\,\|\widehat{\mathbb{T}}_a - \mathbb{T}_a\|_F.
\end{equation*}
where we defined constant \(L_1\coloneq \frac{4 (1 - \epsilon)}{\epsilon^2}\), for which it also holds \(L_1 = (1 - \eta) L\).
\end{corollary}
\begin{proof}
The proof of this corollary easily follows by using the bound in~\eqref{eq:beliefBoundwithRhoBeliefDep} on \(t=1\) and having that \(b_0 = \widehat{b}_0\).
\end{proof}

\vspace{0.5cm}

\section{AUXILIARY RESULTS FOR THE PROOFS OF LEMMA~\ref{lemma:combinedActionDistr} AND THEOREM~\ref{theorem:algorithmBound}}\label{appendix:usefulLemmas}
This section is devoted to the presentation of different useful results that are used throughout the work.
\vspace{0.3cm}

The first one is taken from~\citet{zhou2021regime} and provides a bound on the maximum span \(span(v)\) of the bias function appearing in the Bellman Equation~\eqref{eq:Bellman}. 

\begin{proposition}[Uniform bound on the bias span from~\citet{zhou2021regime}]\label{prop:uniformBoundBias}
	Let us assume to have a POMDP instance that can be rewritten as a belief MDP. If Assumption~\ref{ass:minElem} holds, then for \(\rho,v\) satisfying the Bellman Equation~\eqref{eq:Bellman}, we have the span of the bias function \(span(v):=\max_{b \in \mathcal{B}}v(b)-\min_{b \in \mathcal{B}}v(b)\) is bounded by \(D(\epsilon)\), where:
	\begin{align*}
		D(\epsilon):= \frac{8 \Big(\frac{2}{(1 -\alpha)^2} + (1 + \alpha)\log_{\alpha} \big( \frac{1 - \alpha}{8}\big) \Big)}{1 - \alpha}, \qquad \text{with} \qquad \alpha = \frac{1 - 2\epsilon}{1 - \epsilon} \in (0,1).
	\end{align*}
\end{proposition}
For all bias functions \(v\) associated with a belief MDP generated from a POMDP \(\mathcal{Q}\), this proposition ensures that \(span(v)\) is bounded by \(D=D(\epsilon/2)\).

\vspace{0.3cm}

\begin{lemma}[Lemma A.1 in~\citet{ramponiTruly2020}]\label{lemma:trulyBatch}
	Let \(\bm{x}, \bm{y} \in \mathbb{R}^d\) any pair of vectors, then it holds that:
	\begin{equation*}
		\left\lVert \frac{\bm{x}}{\lVert \bm{x} \rVert_2} - \frac{\bm{y}}{\lVert \bm{y} \rVert_2} \right\rVert_2 \le \frac{2 \lVert \bm{x} -\bm{y} \rVert_2}{\max\{ \lVert\bm{x}\rVert_2,\lVert\bm{y}\rVert_2\}}
	\end{equation*}
\end{lemma}
\vspace{0.5cm}

The following result instead shows the relation between vectors which are obtained by the aggregation of higher-dimensional ones.
\vspace{0.3cm}
\begin{lemma}[Aggregation Lemma in~\citet{russo2024efficient}]\label{lemma:aggregation}
	Let \(\bm{M}\) be a matrix of dimension \(X \times Y\) and have positive values. Let \(\widehat{\bm{M}}\) be an estimation of \(\bm{M}\). Let now \(\bm{c}\) be a vector of dimension \(X\) obtained by summing all the elements of \(\bm{M}\) along the second dimension, such that \(\bm{c}(i) = \sum_{j=1}^J \bm{M}(i,j)\) and let \(\widehat{\bm{c}}\) be a vector obtained with the same procedure from \(\widehat{\bm{M}}\). Then we will have:
	\begin{align*}
		\|\widehat{\bm{c}} - \bm{c}\|_2 \le \sqrt{Y} \|\widehat{\bm{M}} - \bm{M}\|_F
	\end{align*}
\end{lemma}

\vspace{0.5cm}

\begin{lemma}[Lemma 19 in~\citet{jaksch2010near}]
\label{lemma:boundJaksch}
	For any sequence of numbers \(y_0, \dots, y_{n-1}\) with \(0 \le y_k \le Y_k\) and \(Y_k\coloneq \max \{1, \sum_{i=0}^{k-1} y_i\}\):
\begin{align*}
    \sum_{k=0}^{n-1} \frac{y_k}{\sqrt{Y_k}} \le \big( \sqrt{2} + 1 \big) \sqrt{Y_n}.
\end{align*}
\end{lemma}

\vspace{0.5cm}

\begin{lemma}\label{lemma:minActionDist}
	Let a policy \(\pi \in \mathcal{P}\) interact with a POMDP instance satisfying Assumption~\ref{ass:minElem} and let \(\bm{d}^{(a)}_{S^2} \in \Delta({\mathcal{S}^2})\) denote the distribution induced on a pair of consecutive states \((S_t,S_{t+1})\) conditioned on event \(A_t=a\).\\
	Let us denote with \(\bm{d}_{S^2}^{(a)}(s,\cdot) \in \mathbb{R}^S\) the vector containing the different elements \(\bm{d}_{S^2}^{(a)}(s,s')\; \forall s' \in \mathcal{S}\) and let us denote the vector of sum as \(\bm{d}_{S^2}^{(a)}(s)\coloneq \sum_{s' \in \mathcal{S}}\bm{d}_{S^2}^{(a)}(s,s')\). Then, for any state \(s \in \mathcal{S}\), it holds that:
\begin{align*}
	\left\|\bm{d}^{(a)}_{S^2}(s,\cdot)\right\|_2^2 \ge \frac{d_{\min}^{(a)}}{\sqrt{S}},
\end{align*}
with \(d^{(a)}_{\min}\) being the minimum value of the distribution \(\bm{d}^{(a)}_{S^2}\).
\end{lemma}
\begin{proof}
	The result of the lemma derives from the following considerations:
\begin{align*}
		\left\|\bm{d}^{(a)}_{S^2}(s,\cdot)\right\|_2^2 & = \sum_{s' \in \mathcal{S}} \left[\bm{d}^{(a)}_{S^2}(s,s')\right]^2 \\
		& \ge \frac{1}{S} \left( \sum_{s' \in \mathcal{S}}\bm{d}^{(a)}_{S^2}(s,s')\right)^2 \\
		& = \frac{1}{S} \left[\bm{d}^{(a)}_{S}(s)\right]^2 \\
		& \ge \frac{1}{S} \left[d_{\min}^{(a)}\right]^2,
\end{align*}
where the first equality simply derives from the definition of \(\bm{d}_{S^2}^{(a)}(s,\cdot)\), while the first inequality derives from the relation \(\sqrt{X}\left\|\bm{x}\right\|_2 \ge \left\|\bm{x}\right\|_1\) holding \(\forall \bm{x} \in \mathbb{R}^X\).\\
	The second equality instead directly derives from the definition of \(\bm{d}^{(a)}_{S}(s)\). 
	For the last inequality, we have defined \(d_{\min}^{(a)} := \min_{s' \in \mathcal{S}}\bm{d}^{(a)}_{S}(s')\) as a lower bound to the values of the distribution \(\bm{d}^{(a)}_{S}\) induced by policy \(\pi \in \mathcal{P}\).	
\end{proof}

\begin{lemma}(Link between Transition Model and Induced Distribution)\label{lemma:linkTransitionModelAndDist}
	Let a distribution \(\bm{d}_{S^2}^{(a, n,m)} \in \Delta(\mathcal{S}^2)\) be defined on consecutive states, as in Equation~\eqref{eq:fromASStoSS}. Then, the following relation holds:
\begin{align*}
	\mathbb{T}_a(s'|s) = \frac{\bm{d}_{S^2}^{(a,n,m)}(s,s')}{\sum_{s'' \in \mathcal{S}}\bm{d}_{S^2}^{(a,n,m)}(s,s'')} \qquad \forall s,s' \in \mathcal{S}. 
\end{align*}	
\end{lemma}
\begin{proof}
	The result of the lemma easily derives from the following observations.\\
	Starting from the distribution on consecutive states \(\bm{d}_{S^2}^{(a, n,m)}\), we introduce a new distribution \(\bm{d}_{S}^{(a, n,m)} \in \Delta(\mathcal{S})\) defined on a single state and where each of its elements is obtained as follows:
\begin{align}\label{lemma:link01}
	\bm{d}_{S}^{(a, n,m)}(s) = \sum_{s' \in \mathcal{S}} \bm{d}_{S}^{(a, n,m)}(s,s').
\end{align}
The successive key step is to observe that:
\begin{align}\label{lemma:link02}
	\bm{d}_{S}^{(a, n,m)}(s,s') = \bm{d}_{S}^{(a, n,m)}(s) \mathbb{T}_a(s'|s)
\end{align}
which holds for the Markovianity of the problem since the probability of the next state \(s'\) given the current state \(s\) and the action \(a\) is determined by the transition model.\\
By using both~\ref{lemma:link01} and~\ref{lemma:link02}, we obtain:
	\begin{align*}
		\frac{\bm{d}_{S^2}^{(a,n,m)}(s,s')}{\sum_{s'' \in \mathcal{S}}\bm{d}_{S^2}^{(a,n,m)}(s,s'')} = \frac{\bm{d}_{S}^{(a, n,m)}(s) \mathbb{T}_a(s'|s)}{	\bm{d}_{S}^{(a, n,m)}(s)} = \mathbb{T}_a(s'|s) \qquad \forall s,s' \in \mathcal{S}.
	\end{align*}
which completes the proof.
\end{proof}

\vspace{0.5cm}

\section{SIMULATION DETAILS}\label{appendix:simulationDetails}
This section is devoted to providing details about the numerical experiments reported in the main paper. All the reported experiments have been run using 88 Intel(R) Xeon(R) CPU E7-8880 v4 @ 2.20GHz CPUs and 94 GB of RAM.

\vspace{0.5cm}

\subsection{Generation of Transition and Observation Models}\label{appendix:generation}
The instances used in the various experiments have been generated in a random way and, in a successive step, the following modifications are applied: 
\begin{itemize}
    \item concerning each action transition matrix \(\mathbb{T}_a\), the generated ones are such that their minimum transition probability is at least \(\epsilon=1/(20S)\).
    \item for the observation model, for each pair of states and actions, we set a specific observation that will be drawn with higher probability in order to avoid having too much stochasticity in the reward distributions and ensure a diverse observation distribution among states.
\end{itemize}

\vspace{0.5cm}

\subsection{Estimation Error of Transition Matrix}
As reported in the main paper, the characteristic of the considered POMDP instances are: for the left plot \(S=5\) states, \(A=4\) actions and \(O=8\) observations, while for the plot on the right we have \(S=10\) states, \(A=4\) actions and \(O=16\) observations.\\
Instead of using belief-based policies that are optimal in the long horizon but require a planning step, we opted for policies choosing the action that maximizes the instantaneous expected reward based on the current belief state. In order to also select sub-optimal actions (since in this experiment we are not interested in the cumulated reward), we make these policies stochastic, with each action having a minimum probability \(\iota=0.15\) of being selected. Each policy updates its belief based on an internal transition model which is independent and different from the real transition model of the POMDP. After running each policy for \(10^4\) steps, we change the internal transition model used for the belief update: this will also change the distribution induced by the policy. We apply this methodology to both POMDP instances. We repeat each experiment 10 times and we report in the plot the average result for each transition matrix and a \(95\%\) confidence interval.

In order to provide a more detailed analysis of the results, we report in Figure~\ref{fig:estimationErrorWithLegend} the same plot appearing in Figure~\ref{fig:estimationError} and provide details about the characteristics of the different actions in Table~\ref{tab:EstimationError}. In particular, the table shows: (i) the minimum singular value \(\sigma_{\min}(\mathbb{O}_a)\) associated with each action observation model used in the experiment; (ii) the average number of times each action is chosen along the experiment.

By analyzing the Figure, we can observe that actions with low \(\sigma_{\min}(\mathbb{O}_a)\) typically have higher estimation error in the transition model \(\mathbb{T}_a\) since they require a larger amount of samples. This aspect can indeed be observed in the dependencies of problem parameters appearing in Lemma~\ref{lemma:combinedActionDistr}.\\
Of course, this aspect is mitigated when the number of pulls increases. For example, Action 2 on the right figure presents a low \(\sigma_{\min}(\mathbb{O}_a)\) but the high number of pulls makes the estimation error lower.

\vspace{0.5cm}

\subsection{Regret Experiments}
The experimental results on the regret have been done in the following way. First of all, we consider a POMDP instance with \(S=3\) states, \(A=4\) actions and \(O=4\) observations generated as described in Section~\ref{appendix:generation}. We run each experiment 10 times using the total horizon of approximately \(T \approx 4*10^5\) timestamps. The methodology employed for this set of experiments is similar to the one adopted in the work of~\citet{russo2024efficient}.

\begin{figure}[t]
\centering
\includegraphics[scale=1]{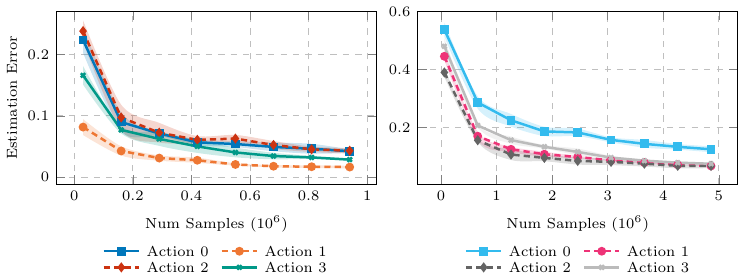}
\captionof{figure}{Error in Frobenious Norm of the Different 
  Action Transition Matrices under two POMDP Instances (10 runs, 95 \%c.i.).}
\label{fig:estimationErrorWithLegend}
\end{figure}

\begin{table}[t]
\centering
\caption{Table representing the minimum singular value of the action observation matrices  and the average (\(\pm\) std) number of pulls of the actions in the experiments of Figure~\ref{fig:estimationErrorWithLegend}.}
\begin{tabularx}{\textwidth}{|>{\centering\arraybackslash}p{3.2cm}|>{\centering\arraybackslash}X|>{\centering\arraybackslash}X|
>{\centering\arraybackslash}X|>{\centering\arraybackslash}X|}
\hline
   \parbox[c][0.5cm]{3.5cm}{} & \textbf{Action \(0\)} & \textbf{Action \(1\)} & 
 \textbf{Action \(2\)} & 
 \textbf{Action \(3\)}\\ \cline{1-5}
\end{tabularx}

\vspace{0.3cm}

\makebox[\textwidth][l]{\quad \quad \textbf{Left Figure}}
\begin{tabularx}{\textwidth}{|>{\centering\arraybackslash}p{3.2cm}|>{\centering\arraybackslash}X|>{\centering\arraybackslash}X|
>{\centering\arraybackslash}X|>{\centering\arraybackslash}X|}
\hline
    
  \parbox[c][0.8cm]{3.5cm}{\centering
 {\(\sigma_{\min}(\mathbb{O}_a)\)}} & 0.214 & 0.248 & 0.164 & 0.670 \\ 
 \cdashline{1-5}
  \parbox[c][0.8cm]{3.3cm}{\centering
 {Number of Pulls}} & \(149994.2\,(\pm 162.5)\)  & \(550202.6\,(\pm 582.2)\) & \(150093.6\,(\pm 261.9)\) & \(149709.6\,(\pm 508.7)\) \\  \cline{1-5}
\end{tabularx}

\vspace{0.3cm}

\makebox[\textwidth][l]{\quad \quad \textbf{Right Figure}}

\begin{tabularx}{\textwidth}{|>{\centering\arraybackslash}p{3.2cm}|>{\centering\arraybackslash}X|>{\centering\arraybackslash}X|
>{\centering\arraybackslash}X|>{\centering\arraybackslash}X|}
\hline
  
  \parbox[c][0.8cm]{3.5cm}{\centering
 {\(\sigma_{\min}(\mathbb{O}_a)\)}} & 0.057 & 0.155 & 0.078 & 0.126 \\ 
 \cdashline{1-5}
  \parbox[c][0.8cm]{3.3cm}{\centering
 {Number of Pulls}} & \(1009711.2\,(\pm 706.7)\)  & \(749978.4\,(\pm 341.1)\) & \(2356435.0\,(\pm 1290.5)\) & \(883875.4\,(\pm 457.8)\) \\  \cline{1-5}
\end{tabularx}

\vspace{0.1cm}
\label{tab:EstimationError}
\end{table}

In particular, the planning task is executed by discretizing the belief space. We can thus solve the Bellman Equation by adapting the Extended Value Iteration algorithm~\citep{jaksch2010near} to the discretized state space.

Inspired by similar works such as~\citet{Azizzadenesheli2016Reinforcement} and~\citet{russo2024efficient}, the theoretical bounds are replaced by smaller values. This approach is commonly employed when performing experimental comparisons in these settings and it mostly translates into a regret with bigger multiplicative constants or a result holding with smaller probability.

\vspace{0.3cm}
\paragraph{SEEU algorithm} In the experiments, the classical Spectral Decomposition approach is modified to make the comparison between the approaches more fair. Following the procedure highlighted in~\citet{russo2024efficient}, we have that:
\begin{itemize}
    \item The matrices used by the Spectral Decomposition approach are not updated based on the realizations of the observation received when choosing an action, but we directly provide the matrix with the probabilities defining the real observation distribution associated with the chosen action and the underlying state. This caveat helps the estimation of both the transition and the observation model since it removes the noise given by the realizations of the observations.
    \item The computation of the optimistic policy for the SEEU algorithm is done by providing the real observation model (together with the estimated transition model) to the Extended Value Iteration algorithm.
\end{itemize}
The parameters used for the SEEU algorithm are $\tau_1 = 8000$ and $\tau_2 = 20000$, which are used to determine the length of the exploration and the exploitation phase respectively.

\vspace{0.3cm}
\paragraph{PSRL-POMDP} In order to implement this algorithm, we opted for the particle filter approach, commonly used in the Bayesian setting. The particle filter strategy does not offer guarantees in terms of consistency but allows updating the model parameters in a tractable manner. We chose this approach since for the moment no consistent estimators for the latent variable model are present in the Bayesian setting. By recalling the parameters of the algorithm described in~\citep{jahromi2022online}, we set:
\begin{itemize}
    \item $SCHED(t_k, T_{k-1}) = t_k + T_{k-1}$ with $t_k$ representing the length of the $k$-th episode;
    \item Let us consider that $n_t(a)$ counts the number of times action $a$ has been pulled up to time $t$. By following the approaches proposed in their work, we set $\widetilde{m}_t(s,a)=n_t(a)$ with $\widetilde{m}_t(s,a)$ being an upper bound to the expected number of times the pair $(s,a)$ has been encountered up to time $t$. 
    \item We use $N=100$ particles for each experiment, while updates of the particles are triggered when the \emph{effective sample size} (ESS) associated with their weights goes below $30$.
\end{itemize} 

\vspace{0.3cm}
\paragraph{OAS-UCRL} Concerning the OAS-UCRL algorithm, we set a minimum action probability \(\iota=0.025\) for all the actions. We chose this value since higher values would have incurred into a higher regret over time. Finally, the initial length of the episode has been set to \(T_0=2500\).\\
Differently from the procedure suggested in~\citet{russo2024efficient} which considers non-overlapping pairs of consecutive elements, we adapted the approach to consider overlapping ones. This modification preserves the theoretical guarantees of the approach and merely translates to adding a multiplicative factor that accounts for the dependency of overlapping samples.

\vspace{0.3cm}
\paragraph{\emph{Action-wise} OAS-UCRL} For our approach, which does not need many parameters to be set, we provided an initial length of the episode of \(T_0=2500\), analogously as the OAS-UCRL case.

\vspace{0.5cm}

\subsection{Ablation Study on Minimum Action Probability of OAS-UCRL}

In this set of experiments, we compare the performance of the OAS-UCRL algorithm when different values for the minimum action probability \(\iota\) are used. The simulations are executed on two different POMDP instances, each one having \(S=3\) states, \(A=3\) actions, and \(O=3\) observations.

The results in terms of regret for each of the two POMDP instances are shown in Figure~\ref{fig:ablation_plot}.\\
For each instance, we run the OAS-UCRL algorithm using as minimum action probability the values \(\iota=0.01\), \(\iota=0.03\), and \(\iota=0.05\). 

As highlighted in the figure, there exists a trade-off for the value of \(\iota\). In particular, when it has higher values, the amount of exploration increases, which results in having better model estimates but may lead to higher regret since suboptimal actions are chosen more frequently. This aspect can be observed on the right plot.

Contrarily, when \(\iota\) has lower values, it may result in a lower regret since suboptimal actions are chosen less frequently, but this aspect could also lead to imprecise model estimates, which could in turn increase the suffered regret, as observed from the left plot in Figure~\ref{fig:ablation_plot}. 

In both plots, we observe the superiority of the Action OAS-UCRL algorithm, whose exploration is driven only by the optimistic approach. In addition, the ability of Action OAS-UCRL to reuse samples across episodes allows for better model estimates, which contributes to experiencing a lower regret.  

\begin{figure}[t]
\centering
\includegraphics[scale=1]{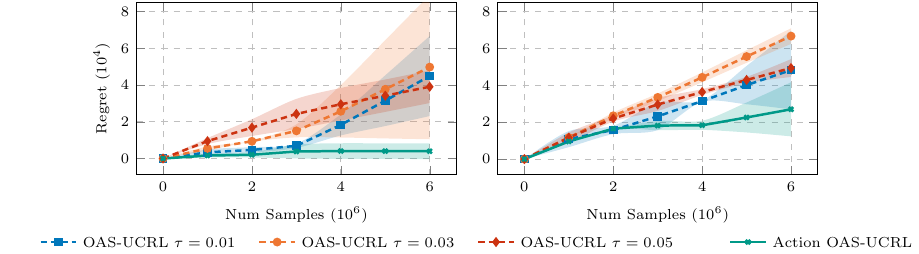}
\captionof{figure}{Regret Experiments on two Different POMDP Instances Comparing Action OAS-UCRL and OAS-UCRL with Different Values of \(\iota\) (10 runs, 95 \%c.i.).}
\label{fig:ablation_plot}
\end{figure}

\vspace{2.cm}

\subsection{Ablation Study on Sample Reuse Strategy of Action OAS-UCRL}
The objective of this set of experiments is to show the effect of reusing all samples collected during the different episodes against the case where only samples from the last episodes are used for model estimation (as done in OAS-UCRL~\citep{russo2024efficient}). 

As shown in Lemma~\ref{lemma:combinedActionDistr}, the estimation error of each transition matrix scales with the number of times the action has been pulled during the different episodes. Hence, using samples collected from all the episodes leads to lower estimation error with respect to using only samples from the last episode. 

Figure~\ref{fig:ablation_plot_noSR} presents the results in terms of regret of the standard AOAS-UCRL algorithm and a variant that only uses samples collected from the last episode. Concerning the variant, since AOAS-UCRL may not select all actions during each episode, we use uniform transition matrices for actions that are not chosen during the last episode.

The experiments in Figure~\ref{fig:ablation_plot_noSR} are run on three different POMDP instances, each one having \(S=3\) states, \(A=5\) actions and \(O=3\) observations. We observe that reusing all samples generally leads to better performances overall.

In particular, depending on the specific POMDP instance, this advantage can be more or less evident. Indeed, the two instances on the left show more evidently the benefit of reusing all samples, differently from the instance on the right which shows similar performances. It can indeed be the case that even if the estimated model presents higher error, the policies computed using the two strategies are similar and lead to comparable performances.

Another remark, more evident in the plot in the center, is that the confidence intervals for the strategy without sample reuse are larger since the estimates may vary more across the different runs since they rely on less samples.

\begin{figure}[t]
\centering
\includegraphics[scale=0.85]{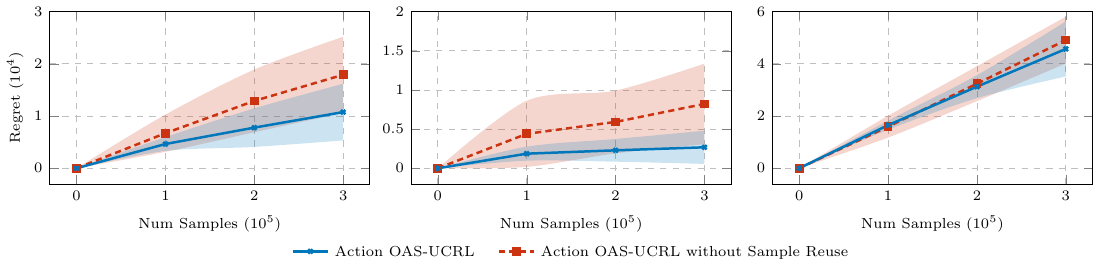}
\captionof{figure}{Regret Experiments comparing the Regret Performance of the standard AOAS-UCRL Algorithm Against the case with No Sample Reuse Across Episodes (10 runs, 95 \%c.i.).}
\label{fig:ablation_plot_noSR}
\end{figure}

\end{document}